%% file: main.tex
\documentclass{article}
\usepackage{iclr2027_conference,times}
\usepackage[T1]{fontenc}  % real glyphs for _ in \texttt{} code names (OT1 renders them as rules)
\input{math_commands.tex}

\input{generated/numbers.tex}

\input{generated/prefix_numbers.tex}
\input{generated/history_numbers.tex}

\input{generated/history_extensions_numbers.tex}
\input{generated/relative_error_numbers.tex}

\input{generated/review_numbers.tex}

\input{generated/review_ablation_numbers.tex}
\usepackage{hyperref}
\hypersetup{
  hidelinks,
  pdftitle={Shared Autoregressive Context Can Distort Relationships in Synthetic Data},
  pdfauthor={Thomas S. Robinson}
}
\usepackage{url}
\usepackage{amsmath,amssymb,amsthm}
\usepackage{booktabs}
\usepackage{graphicx}
\usepackage{multirow}
\usepackage{xcolor}
\usepackage{listings}

\lstdefinestyle{prompt}{%
  basicstyle=\ttfamily\scriptsize,
  breaklines=true,
  breakindent=0pt,
  breakautoindent=false,
  columns=fullflexible,
  keepspaces=true,
  frame=single,
  framesep=4pt,
  xleftmargin=4pt,
  xrightmargin=2pt,
  aboveskip=8pt,
  belowskip=8pt,
  literate={—}{{\normalfont\textemdash}}1,
}

\theoremstyle{plain}
\newtheorem{proposition}{Proposition}
\newtheorem{corollary}{Corollary}

\theoremstyle{definition}
\newtheorem{definition}{Definition}
\theoremstyle{remark}

\title{Shared Autoregressive Context Can Distort\\ Relationships in Synthetic Data}

\author{Thomas S. Robinson\\
\normalfont London School of Economics and Political Science\\
\normalfont\texttt{t.robinson7@lse.ac.uk}}

\iclrfinalcopy
\begin{document}

\maketitle
\lhead{Preprint}

\begin{abstract}
  \input{sections/abstract}
\end{abstract}

\input{sections/intro}
\input{sections/fig_captions}

\input{sections/design}

\input{sections/results}

\input{sections/correction}

\input{sections/related}

\input{sections/limitations}

\input{sections/statements}

\bibliography{references}
\bibliographystyle{iclr2027_conference}

\appendix
\input{sections/app_request_design}
\input{sections/app_history_channel}
\input{sections/app_downstream}
\input{sections/app_correction}
\input{sections/app_review_diagnostics}
\input{sections/app_temperature_prompt}
\input{sections/mechanism}
\input{sections/app_proofs}
\input{sections/app_audit_results}
\input{sections/app_current}
\input{sections/app_prompts}
\input{sections/app_campaign}
\input{sections/app_new_results}
\input{sections/app_marginal_controls}
\input{sections/app_followup_checks}
\input{sections/app_prefix_intervention}

\end{document}

%% file: math_commands.tex
\usepackage{amsmath,amsfonts,bm}

\def\eqref#1{equation~\ref{#1}}
\def\1{\bm{1}}

\DeclareMathAlphabet{\mathsfit}{\encodingdefault}{\sfdefault}{m}{sl}
\SetMathAlphabet{\mathsfit}{bold}{\encodingdefault}{\sfdefault}{bx}{n}

\newcommand{\E}{\mathbb{E}}

\newcommand{\Var}{\mathrm{Var}}

\newcommand{\Cov}{\mathrm{Cov}}

%% file: generated/numbers.tex
\newcommand{\InflatedPercent}{98.2}
\newcommand{\SignErrorPercent}{19.1}
\newcommand{\NoYearAgeRho}{0.730}

%% file: generated/prefix_numbers.tex
\newcommand{\PrefixDelta}{+0.503}
\newcommand{\PrefixDonorCentered}{+0.582}
\newcommand{\PrefixIntervalLow}{0.429}
\newcommand{\PrefixIntervalHigh}{0.641}
\newcommand{\PrefixMatchedTriplets}{600}
\newcommand{\PrefixBootstrapValid}{5{,}000}
\newcommand{\PrefixValidResponses}{1{,}800}
\newcommand{\PrefixTruncations}{0}
\newcommand{\PrefixRepZero}{+0.512}
\newcommand{\PrefixRepOne}{+0.489}
\newcommand{\PrefixRepTwo}{+0.512}
\newcommand{\PrefixSourceNatural}{0.487}
\newcommand{\PrefixSourcePermuted}{0.047}
\newcommand{\PrefixSourceAligned}{0.966}

%% file: generated/history_numbers.tex
\newcommand{\HistoryQwenHealthPositive}{0.963}
\newcommand{\HistoryQwenHealthNegative}{-0.670}
\newcommand{\HistoryQwenSurveyPositive}{0.985}
\newcommand{\HistoryQwenSurveyNegative}{-0.647}
\newcommand{\HistoryLlamaHealthPositive}{0.837}
\newcommand{\HistoryLlamaHealthNegative}{-0.329}
\newcommand{\HistoryLlamaSurveyPositive}{0.725}
\newcommand{\HistoryLlamaSurveyNegative}{0.041}

%% file: generated/relative_error_numbers.tex
\newcommand{\RelDownstreamACSLlamaLearned}{42}
\newcommand{\RelDownstreamACSLlamaPreserved}{53}
\newcommand{\RelDownstreamACSQwenLearned}{293}
\newcommand{\RelDownstreamACSQwenPreserved}{136}
\newcommand{\RelDownstreamESSLlamaLearned}{31}

\newcommand{\RelDownstreamESSQwenLearned}{45}

\newcommand{\RelDownstreamNHANESLlamaLearned}{0.18}
\newcommand{\RelDownstreamNHANESLlamaPreserved}{17}
\newcommand{\RelDownstreamNHANESQwenLearned}{27}

\newcommand{\RelFrontierESSLlama}{122}
\newcommand{\RelFrontierESSQwen}{51}
\newcommand{\RelFrontierNHANESLlama}{54}
\newcommand{\RelFrontierNHANESQwen}{56}

\newcommand{\RelFrontierSeedsESSLlama}{116\text{--}131}
\newcommand{\RelFrontierSeedsESSQwen}{51\text{--}52}

\newcommand{\RelHistoryBMIFive}{95}
\newcommand{\RelHistoryBMITwo}{99}
\newcommand{\RelHistoryMarginalCost}{8\text{--}11}

\newcommand{\RelOriginalHistoryGain}{47}

\newcommand{\RelPrimaryLlama}{114\text{--}127}
\newcommand{\RelPrimaryQwen}{48\text{--}58}

%% file: generated/review_numbers.tex
\newcommand{\ReviewQwenOne}{0.183}
\newcommand{\ReviewQwenBatch}{0.278}
\newcommand{\ReviewQwenDelta}{0.095}
\newcommand{\ReviewQwenPairCount}{155--159}
\newcommand{\ReviewLlamaOne}{0.187}
\newcommand{\ReviewLlamaBatch}{0.415}
\newcommand{\ReviewLlamaDelta}{0.228}
\newcommand{\ReviewLlamaPairCount}{148--154}

%% file: sections/abstract.tex
Large language models can generate several records within one autoregressive completion, making earlier answers available as context for later records. This paper shows that such shared-completion batching can distort relationships among variables in the resulting synthetic data, using controlled tests on synthetic survey respondents. In a matched experiment on $2{,}000$ European Social Survey profiles, generating ten rather than one respondent per request increases mean absolute error in within-country correlations by $\RelPrimaryQwen\%$ for Qwen3.8-27B and $\RelPrimaryLlama\%$ for Llama-3.3-70B-Instruct across three seeds, holding profiles, examples, questions and decoding parameters fixed. The distortion primarily reflects exaggerated relationship strength, while retaining substantial agreement with the human ordering of correlations. Controlled interventions establish answer history as a causal channel: re-pairing the same preceding values, with profiles and marginal distributions fixed, changes correlations among subsequently generated responses. Hiding preceding answers reduces correlation error in the tested settings but worsens marginal accuracy. Exploratory corrections across social-attitude, health and economic data likewise show that lower correlation error can coexist with worse marginal distributions and regression estimates. Request construction is therefore part of the data-generating process, and synthetic-data validity must be evaluated against the analyses the generated data are intended to support.

%% file: sections/intro.tex
\section{Introduction}\label{sec:intro}

Generating several records in one autoregressive completion can change the statistical
relationships in the resulting synthetic data. A large language model (LLM) can answer
for ten people in one completion, or for the same ten profiles in ten separate completions.
In the shared completion, earlier answers become context for later ones and can influence
the population statistics estimated from those records. Here, \emph{shared-completion
batching}, or batching below, refers to
respondents sharing a completion, independently of how a serving system schedules requests.

For a generator intended to sample respondents independently conditional on their profiles,
a useful desideratum is \emph{request-partition invariance}: splitting the same profiles
across requests should not materially change the population quantities the sample is
intended to estimate. This criterion allows ordinary variation between random draws and
concerns the distribution of repeated samples; autoregressive generation does not guarantee it.

Synthetic survey respondents provide a controlled testbed for this criterion. Silicon
sampling generates answers for people described by demographic profiles
\citep{argyle2023out}, and applied studies sometimes request dozens of respondents together
\citep{jahn2026refielding,miklian2026stochastic}. Human reference samples support comparisons of
both individual-variable distributions and relationships among variables: matching income
and trust distributions alone does not establish whether income predicts trust within a country.

The main experiment holds European Social Survey (ESS) profiles, examples, questions and
decoding settings fixed while changing request size. Generating ten respondents rather
than one raises mean absolute error across $171$ within-country correlations by
$\ReviewQwenDelta$ for Qwen3.8-27B and $\ReviewLlamaDelta$ for Llama-3.3-70B-Instruct,
with relative increases of $\RelPrimaryQwen\%$ and $\RelPrimaryLlama\%$ across three
seeds (Figure~\ref{fig:request_design}). The main change is exaggerated relationship
strength; batched correlations still retain substantial agreement with the human ordering.

Controlled interventions then test the influence of answer history directly. In health and political-trust templates,
re-pairing exactly the same preceding values changes correlations among newly generated
respondents in both models. The supplied profiles and each variable's distribution stay
fixed; only the pairing changes. This establishes answer history as a causal channel
through which shared context can alter later relationships
(Figure~\ref{fig:history_channel}), while leaving the model's internal computation open.
Hiding earlier answers provides a complementary test against human data: correlation
error falls even when only one respondent precedes the target, although marginal accuracy worsens.

These experiments isolate request size and manipulate history, extending earlier evidence
of distorted synthetic relationships
\citep{bisbee2024synthetic,lukauskas2026plausible,ma2026omnivorousness}. Exploratory
corrections then examine whether improved correlation accuracy makes the data more useful.
Across social-attitude, health and economic data, gains can coexist with worse marginal
distributions and regression estimates. Request construction belongs in the generator
specification, and its evaluation must address the quantities an intended analysis requires.

%% file: sections/fig_captions.tex
% Matched request experiment; the descriptive panel appears in its appendix.
\begin{figure}[!t]
\centering
\includegraphics[width=\textwidth]{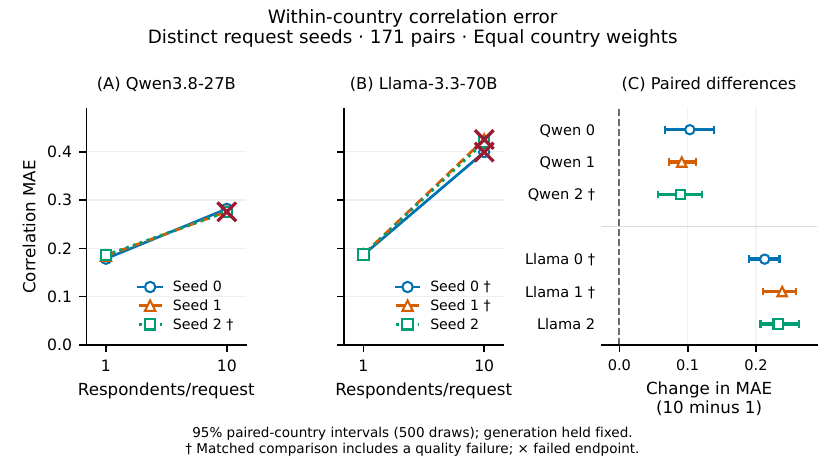}
\caption{\textbf{Sharing a completion increases correlation error in the matched ESS experiment.}
(A,B) Mean absolute error against human data across $171$ within-country correlations,
with equal country weights and the same accepted respondents in all four conditions.
Each line shows one seed. Error increases by $\RelPrimaryQwen\%$ for Qwen and
$\RelPrimaryLlama\%$ for Llama. (C) The increase from one to ten respondents, with
$95\%$ paired-country percentile bootstrap intervals ($500$ resamples). These intervals
condition on the generated answers and twenty selected countries; they omit uncertainty
over new seeds and other countries. Conditions below $99\%$ validity are marked as
quality failures. Variable pairs are not independent replications.}
\label{fig:request_design}
\end{figure}

%% file: sections/design.tex
\section{Data and evaluation}\label{sec:design}

\paragraph{Data.}
Three human surveys provide demographic profiles and reference answers, beginning with
$2{,}000$ European Social Survey (ESS) respondents in the main experiment: one hundred
in each of twenty countries. Their ten background characteristics condition the generation
of nineteen numeric responses covering trust, political attitudes, well-being, income,
age and interview year.

The 2017--2018 National Health and Nutrition Examination Survey (NHANES)
\citep{nchs2020nhanes} extends the study to U.S. adult health, while the 2022 American
Community Survey (ACS) \citep{census2023acs} supplies income and work data for employed
Massachusetts adults. Each contributes $800$ people, with one hundred in each of eight
groups defined by two sex categories and four age bands. The model receives each person's
recorded sex, age band, race/ethnicity, education and marital status;
Appendix~\ref{app:controlled_sources} lists the outcomes, age bands and eligibility rules.

\paragraph{Matched comparison.}
For each target person, the model receives a demographic profile alongside three examples
pairing other people's profiles with their real answers. The target's own survey answers
remain hidden, and a separate set of human respondents provides the evaluation reference.
The model generates answers for the same targets one at a time or ten at a time, keeping
profiles, examples, question order and non-seed generation settings fixed. Each request
receives a distinct random seed, with three runs per model at temperature $1$.
Appendices~\ref{app:request_design} and~\ref{app:controlled_sources} give the sampling and
request-construction details.

\paragraph{Invalid answers.}
The validation procedure rejects invalid or truncated requests in full, without repair
or regeneration, and comparisons use profiles accepted across the relevant conditions. A condition has a
\emph{quality failure} if fewer than $99\%$ of requests are valid or any are truncated.
The analysis retains such conditions, since matching accepted profiles cannot recover rejected answers.

\paragraph{Accuracy.}
The primary measure concerns relationships among variables across people within the same
country or demographic group. Subtracting each group's mean from each variable before calculating
correlations removes differences in average response levels between groups.
For a real or synthetic population $P$, with $c$ denoting the group,
\begin{equation}\label{eq:within}
 r_W^P=\mathrm{Corr}_P\big(Y_1-\E_P[Y_1\mid c],Y_2-\E_P[Y_2\mid c]\big).
\end{equation}
With groups receiving equal weight, correlation MAE ($A$) averages the absolute differences
between synthetic and human correlations. Further checks distinguish strength from pattern
by comparing the ordering of correlations, a zero-correlation baseline, and an \emph{oracle}
adjustment that chooses one multiplier in $[0,1]$ to minimize error against human answers.
The oracle diagnoses how much error simple weakening can remove; it uses evaluation
answers and is not an available correction (Appendix~\ref{app:review_metrics}).

Marginal Wasserstein error measures how far values must move, on average, to match each
human variable's distribution. Dividing by the variable's range ($M$) or human standard
deviation ($M_{\rm SD}$) puts errors on a common scale; the latter is emphasized for ACS,
where broad income ranges can hide errors (Equation~\ref{eq:marginal_error}). Fixed
regression tasks assess accuracy for particular analyses (Appendix~\ref{app:downstream}).

\paragraph{Uncertainty.}
Seed ranges describe variation across the three runs, while bootstrap intervals assess
sensitivity to resampling countries, profiles or answer orders, depending on the experiment.
Each draw keeps the compared conditions paired and uses the existing generated answers
and human reference; separate checks resample the human reference itself.
Appendix~\ref{app:review_uncertainty} and the study-specific appendices specify which
sources of uncertainty each interval covers.

%% file: sections/results.tex
\section{Shared-completion batching and answer history}\label{sec:distortion}

\subsection{Ten-person requests exaggerate relationship strength}\label{sec:control}

Requesting ten respondents rather than one increases ESS correlation error in every run
of both models under distinct per-request seeds (Figure~\ref{fig:request_design}). On
profiles accepted in all four combinations of request size and seed policy, mean MAE
rises from $\ReviewQwenOne$ to $\ReviewQwenBatch$ for Qwen and from $\ReviewLlamaOne$
to $\ReviewLlamaBatch$ for Llama. The second seed policy resets every request to the
same run seed and serves as a diagnostic; the primary contrast uses distinct seeds.
Across runs, the relative increases are $\RelPrimaryQwen\%$ and $\RelPrimaryLlama\%$,
with error increasing for \ReviewQwenPairCount{} of $171$ pairs in each Qwen run and
\ReviewLlamaPairCount{} in each Llama run. Although these overlapping pairs are not
independent tests, they show that the effect extends across much of the correlation
matrix. All six paired-country intervals for the increase in mean error exclude zero.
The effect also appears in checks that use all retained answers, remove age and year, or reset
every request to the same seed (Appendix~\ref{app:request_design}).

Much of this increase reflects stronger relationships rather than a loss of information
about their ordering. Batching raises mean $|r|$ from $0.360$ to $0.459$ for Qwen and
from $0.355$ to $0.596$ for Llama, compared with $0.193$ in the human reference. Yet
the batched correlations still track the human ordering, with Spearman agreement of
$0.864$ and $0.866$. After the oracle magnitude adjustment applies the best common
multiplier to each set of correlations, one- and ten-person MAE are both approximately $0.10$
(Table~\ref{tab:review_requests}). The main distortion is therefore an exaggeration of
relationship strength in data that retain substantial information about the human pattern.

The increase also survives checks on output validity and format. All $26{,}400$ requests
finish without truncation, although three original ten-person conditions fall below
$99\%$ validity: Qwen seed 2 ($97.5\%$) and Llama seeds 0 and 1 ($98\%$). The effect
holds in the three conditions that pass as well, with matching retaining at least eighty
profiles per country. Constraining generation to the required output format likewise
leaves the effect in ESS and NHANES: ESS mean MAE rises by $\RelFrontierESSQwen\%$
for Qwen and $\RelFrontierESSLlama\%$ for Llama, with corresponding NHANES increases
of $\RelFrontierNHANESQwen\%$ and $\RelFrontierNHANESLlama\%$. ACS provides weaker
evidence because all six ten-person conditions fail validity and Llama's effect varies
in direction across seeds (Table~\ref{tab:structured_frontier}).

The penalty also persists in both models at temperature $0.5$ and after removing the
instruction to ``Preserve realistic individual variation'' at temperature $1$. These
checks show that the effect does not require that wording or the original temperature.
At $1.5$, Llama remains valid and retains the penalty, but many Qwen ten-person requests
fail validation, preventing its joint ESS estimate and NHANES intervals. Higher-temperature
effects therefore depend on the model (Appendix~\ref{app:temperature_prompt}).

This consistent correlation penalty coexists with better individual-variable distributions.
On the correction study's matched evaluation cohorts, batched outputs have lower marginal
Wasserstein error than one-person outputs in all six settings; for ESS/Qwen,
range-normalized $M$ is $0.115$ compared with $0.127$ (Table~\ref{tab:marginal_controls}).
One-person generation is therefore not uniformly more accurate: the choice between
request sizes changes both how individual variables are distributed and how they vary together.

\subsection{Answer history is a causal channel}\label{sec:history_channel}

\paragraph{Changing pairings while keeping values fixed.}
The matched request-size experiment establishes a difference between generation procedures.
A controlled intervention on answer history then tests whether earlier answers can change
later relationships. The model receives ten demographic profiles and completed answers
for the first nine respondents, then generates only the tenth person's answers. The
comparison re-pairs the preceding values while keeping those values, their marginal
distributions and the profiles fixed (Figure~\ref{fig:history_channel}A).

In NHANES, each supplied respondent has a BMI and a value for daily sedentary minutes.
One version pairs higher BMI with more sedentary time, producing a positive relationship;
a second pairs higher BMI with less sedentary time, producing a negative relationship.
A third arranges the same values so that their correlation is zero. Because every version
contains exactly the same BMI and sedentary-time values, these changes alter their pairing
without changing either variable's distribution. The ESS experiment follows the same
procedure using trust in parliament and the judiciary. In both domains, the ten profiles
in each template are identical except for IDs, so rearranging answers also preserves
their agreement with the supplied demographics.

To separate pairing from a particular presentation order, eighteen forward and reverse
rotations place every value equally often in each position. Each order has three
repetitions with generation seeds paired across conditions, giving $12{,}096$
continuations across two models and two domains, including a reference condition that
hides the supplied answers. The correlations describe the newly generated tenth
respondents only; none of the nine supplied answers enters the calculation.

Changing history pairing from positive to negative moves target correlations from
$\HistoryQwenHealthPositive$ to $\HistoryQwenHealthNegative$ for Qwen/NHANES and
from $\HistoryQwenSurveyPositive$ to $\HistoryQwenSurveyNegative$ for Qwen/ESS.
Llama shows the same direction of change, from $\HistoryLlamaHealthPositive$ to
$\HistoryLlamaHealthNegative$ in NHANES and from $\HistoryLlamaSurveyPositive$ to
$\HistoryLlamaSurveyNegative$ in ESS (Figure~\ref{fig:history_channel}A). Every template
shows the same direction of change. After subtracting the mean response for each exact
answer order within each condition, the mean positive-minus-negative contrasts also
remain positive, so differences between orders' average responses cannot explain the
whole result (Appendix~\ref{app:history_channel}). Earlier answer pairings therefore
influence later relationships, although the experiment leaves open whether the model
extracts covariance, extends each variable's sequence separately, or uses another computation.

\begin{figure}[t]
\centering
\input{sections/history_schematic}
\par\medskip
\includegraphics[width=\textwidth]{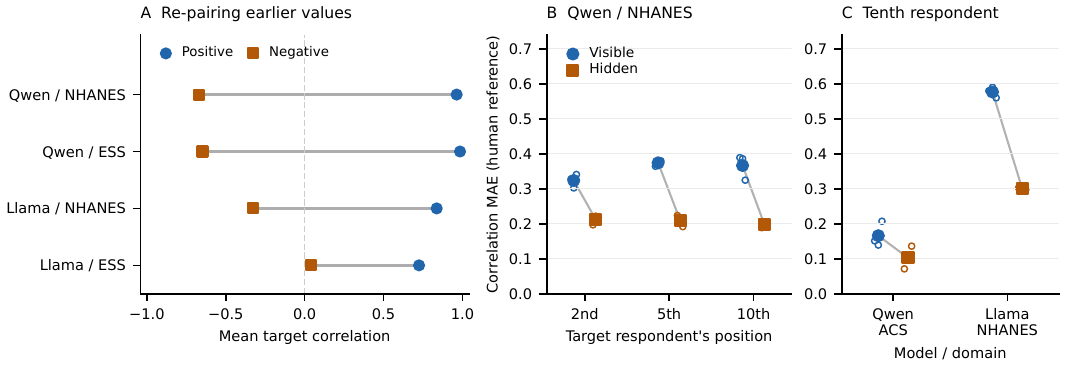}
\caption{\textbf{Manipulating earlier answers changes later relationships.}
The schematic shows the interventions; ranks illustrate pairing, with actual values fixed.
(A) Mean correlation among newly generated tenth respondents, averaged across prompt
templates, after positive or negative pairing of the same preceding values. This
panel measures generated relationships, without comparing them with human data.
Llama uses a separately validated NF4 runtime. (B) Qwen/NHANES correlation MAE for
profiles accepted in all six conditions. (C) Tenth-respondent extensions, matching
profiles across the two conditions. Filled markers average three repetitions;
hollow markers show each repetition. Hiding history skips the earlier answers'
state updates while retaining the target's logical token positions. Mean marginal
error rises in every comparison (Tables~\ref{tab:history_positions}--\ref{tab:history_extensions}).}
\label{fig:history_channel}
\end{figure}

\paragraph{Hiding previously generated answers.}
The pairing experiment establishes that history can change later relationships; a
complementary intervention asks whether access to that history affects accuracy against
human data. Here the model generates the earlier respondents' answers itself. For $400$
Qwen/NHANES profiles, the person whose answers enter the evaluation appears second,
fifth or tenth in a ten-profile prompt, giving that person one, four or nine preceding responses.

Each history then supports two generations of the target's answers using the same prompt
and seed. In the visible condition, the model processes the preceding answers as usual;
in the hidden condition, their tokens cannot update attention keys and values or Qwen's
recurrent and convolution states. Generation resumes at the target's original logical
token positions, preserving its position while removing the earlier answers' contribution
to the model state. The tenth-position experiment also extends to Llama/NHANES and Qwen/ACS.

On Qwen/NHANES profiles accepted in all six position--visibility conditions, hiding
history reduces mean MAE from $0.324$ to $0.213$ at position two, $0.374$ to $0.210$
at five, and $0.367$ to $0.197$ at ten. The corresponding reductions are $0.577$ to
$0.300$ for Llama/NHANES and $0.166$ to $0.103$ for Qwen/ACS, with lower raw MAE
in every repetition and all paired profile-bootstrap intervals excluding zero. These
gains come with worse individual-variable accuracy: mean marginal error rises in each
setting, including increases of $\RelHistoryMarginalCost\%$ in Qwen/NHANES
range-normalized $M$ across positions and from $0.348$ to $0.350$ in Qwen/ACS
$M_{\rm SD}$. Access to earlier answers thus improves marginal distributions while
worsening correlation accuracy in these comparisons.

The gain at position two also narrows the possible explanations, since a single preceding
respondent cannot define a cross-respondent correlation. History can therefore matter
without the model first extracting a covariance matrix from several preceding respondents.

\paragraph{Relationship strength and pattern.}
The magnitude diagnostics qualify what this reduction in raw error means. In Llama/NHANES,
hiding history slightly worsens both oracle-adjusted MAE ($0.099$ to $0.104$) and
agreement between the correlation vectors ($0.749$ to $0.723$), suggesting that the raw
gain largely reflects weaker correlations. Qwen/NHANES at position ten shows a broader
improvement: oracle MAE falls from $0.081$ to $0.065$ and pattern agreement rises from
$0.848$ to $0.894$ (Table~\ref{tab:review_history}). General weakening can therefore
explain much of the apparent benefit, but does not capture every setting's improvement.

Another possible explanation is that each request raises or lowers all its respondents'
answers together. Subtracting each request's mean, with one-person outputs first grouped
into the same ten-profile blocks, makes the ESS batching penalty larger and leaves a
larger leading correlation-matrix eigenvalue under batching. Differences in request
means alone therefore cannot explain the effect. Appendix~\ref{app:review_mechanism}
reports this check and exploratory regressions relating targets to their predecessors.

\paragraph{Supporting context controls.}
Hiding history changes state updates and leaves a gap in token positions. A matched
Qwen/NHANES follow-up closes that gap: MAE is $0.363$ with visible history, $0.192$
with hidden history at the original positions and $0.193$ at consecutive positions.
The benefit therefore survives removal of the gap. In the same runtime, showing only
the target profile gives lower MAE than showing ten profiles with the target first and
no preceding answers ($0.198$ versus $0.247$). These controls distinguish profile context
and answer access within a common continuation procedure; runtime and design differences
prevent assigning them separate shares of the original batching penalty
(Appendix~\ref{app:context_controls}).

%% file: sections/history_schematic.tex
% Native LaTeX schematic; the ranks illustrate the intervention, not new data.
\begingroup
\scriptsize
\begin{minipage}[t]{0.57\textwidth}
\centering
\textbf{Change only the pairing (A)}\par\smallskip
\setlength{\tabcolsep}{3pt}
\begin{tabular}{@{}lcc@{}}
 & Positive history & Negative history \\
$X$ & $1\,2\,3\,4\,5\,6\,7\,8\,9$ & $1\,2\,3\,4\,5\,6\,7\,8\,9$ \\
$Y$ & $1\,2\,3\,4\,5\,6\,7\,8\,9$ & $9\,8\,7\,6\,5\,4\,3\,2\,1$ \\
\end{tabular}\par\smallskip
Same profiles + either history $\longrightarrow$ generate person 10
\end{minipage}\hfill
\begin{minipage}[t]{0.41\textwidth}
\centering
\textbf{Change access to history (B,C)}\par\smallskip
Same profiles and generated earlier answers\par\smallskip
\fbox{Answers visible} $\longrightarrow$ target\par\smallskip
\fbox{Answers hidden} $\longrightarrow$ target\par\smallskip
Target token positions stay fixed
\end{minipage}
\endgroup

%% file: sections/correction.tex
\section{Dependence fidelity is not enough}\label{sec:correction}\label{sec:tradeoff}

Shared-completion batching and access to earlier answers can improve marginal accuracy
while worsening correlation accuracy. Reversing the correlation distortion therefore
need not improve the data for every use. Exploratory corrections examine this trade-off
by changing dependence after generation and assessing the resulting data. The comparisons
use standard operations, including rank-based dependence adjustment
\citep{iman1982distribution}, as diagnostic tools to test what lower correlation error
establishes about usefulness.

\paragraph{A diagnostic transformation.}
A full dependence map puts batched answers on a common normal-score scale, adjusts their
correlation matrix toward that of separate one-person fitting outputs, and converts the
scores back to answer units using the one-person fitting distributions. Fitting and
evaluation use disjoint whole requests, with matched profiles across request sizes;
fitting excludes human evaluation answers. Because the study reuses previously
examined outputs and human references, all correction results remain exploratory.
Appendix~\ref{app:correction} gives the map, fitting design and comparison methods.

\begin{figure}[t]
\centering
\includegraphics[width=0.90\textwidth]{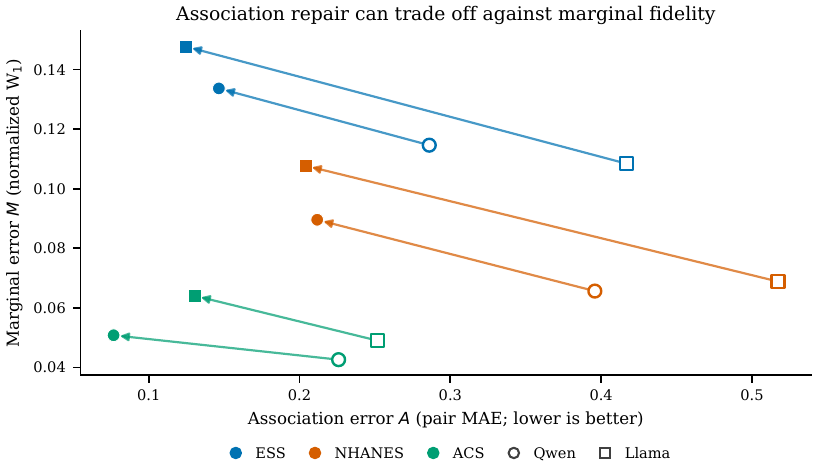}
\caption{\textbf{Lower correlation error can coexist with worse marginal accuracy.}
Open (unfilled) points show errors for the uncorrected ten-person batches; filled points
show errors after applying the full dependence map. Each arrow connects the results
for one domain and model, averaged across three seeds on matched evaluation profiles. Left and down
are better; every arrow moves left and up. $M$ uses range-normalized Wasserstein distance;
SD-normalized ACS results show the same trade-off. ACS quality failures remain included.
This result concerns the full map; exact marginal preservation gives a separate diagnostic.}
\label{fig:tradeoff}
\end{figure}

\paragraph{Improved correlations, worse marginals.}
The full map lowers mean correlation error in all six domain--model settings while
worsening marginal accuracy in each (Figure~\ref{fig:tradeoff}). The result also holds
on the more informative SD-normalized scale for ACS: marginal error rises from $0.376$
to $0.423$ for Qwen and from $0.418$ to $0.525$ for Llama. Much of the correlation gain
requires little pair-specific information. Replacing the full target matrix with one
multiplier for all off-diagonal correlations, fitted using the same one-person outputs,
retains 93--101\% of the MAE improvement (Table~\ref{tab:review_scalar}). In NHANES,
both full maps still have larger correlation errors than predicting zero for every pair.
Lower MAE can thus reward weakening without establishing recovery of the human pattern.

\paragraph{Preserving marginals does not settle usefulness.}
A further control keeps each group's original list of batched values for each variable
and reassigns those values to respondents in the order of their corrected scores.
Every marginal distribution and variance then stays fixed. Correlation error still falls
in $17/18$ full-map seed contrasts, with one worsening ACS/Llama seed
(Table~\ref{tab:correction_seed_contrasts}). Reordering can nevertheless weaken
relationships and change agreement with profiles, so exact marginal preservation does
not establish fidelity for a particular analysis.

Fixed regression tasks test that fidelity directly. ACS/Qwen coefficient MAE rises from
$0.032$ to $0.124$ under the full map and to $0.110$ under the scalar; other settings,
including ESS and Qwen/NHANES, improve (Table~\ref{tab:review_downstream}). The
marginal-preserving variants likewise do not consistently improve regressions
(Appendix~\ref{app:downstream}). A synthetic dataset's usefulness depends on which
features of its joint distribution the analysis requires. Neither reducing correlation
error nor preserving marginal distributions establishes that usefulness by itself.

%% file: sections/related.tex
\section{Related work}\label{sec:related}

Evaluations of LLM-generated synthetic data examine both agreement with human data and
sensitivity to the generation procedure. In survey simulation,
\citet{argyle2023out} introduced silicon sampling using separate prompts for individual
respondents. Subsequent work finds misalignment with demographic groups' opinions \citep{santurkar2023whose},
too little variation and shifted regression coefficients \citep{bisbee2024synthetic}, and
sensitivity to survey presentation \citep{dominguezolmedo2024questioning}. Distorted relationships
are also documented: \citet{lukauskas2026plausible} examine psychometric structure across
thirty-seven models, and \citet{ma2026omnivorousness} study relationships among cultural tastes.

Respondent batching appears in applied studies: \citet[][Web Appendix~A.1]{jahn2026refielding}
request fifty respondents per completion, and \citet[][Appendix]{miklian2026stochastic}
describe batches of thirty to fifty. In qualitative interview simulation,
\citet[][Section~9.2]{sanders2026affordances} request four people's answers together.
\citet{jahn2026refielding} also compare batch, cumulative-memory and seed-and-expand
generation for consistency across repeated surveys and agreement with human data, including
correlations. Building on these comparisons, the matched design in this paper isolates
shared-completion request size while keeping respondent profiles, examples and survey questions fixed.
The history experiments then intervene on the values available to later respondents,
re-pairing them while preserving their marginal distributions and demographic agreement.
These controls connect differences between generation procedures to the effects of
specific changes in request construction and answer history.

Grouping questions is a different design choice from grouping respondents. \citet{lam2026fast}
find advantages, in their setting, to asking one respondent's questions together rather than
in separate calls. In contrast, \citet{choi2026beyond} find better structural agreement from
per-item prompts in their setting, distinguishing agreement in relationships, individual-variable
distributions and individual responses as they evaluate prompting, rectification and
fine-tuning using small human samples. The experiments here address a different source
of variation: the survey questions stay fixed while the number of \emph{respondents}
sharing a request changes. The correction comparison likewise uses generated fitting
answers, with human data providing separate evaluations of correlations and marginal distributions.

Adjusting dependence through ranks while preserving marginals is established
\citep{iman1982distribution}. Standard transformations serve here as diagnostics:
replacing pair-specific targets with one overall adjustment tests how much improvement
requires detailed relationship information, while fixed marginals and regression tasks
test what that improvement means for the resulting data.

Relationships between group averages can differ from relationships among people within
groups \citep{robinson1950ecological,goodman1953ecological,king1997solution}. For example,
a relationship across countries need not hold among people in each country. This motivates
asking whether models transfer relationships from one level to the other, while leaving
that explanation to be tested.
Tabular generators such as CTGAN and TabDDPM are fitted to individual-level training data
\citep{xu2019modeling,kotelnikov2023tabddpm}, unlike prompting a pretrained model with profiles
and a few examples. Prediction-powered inference can correct particular estimators using real
observations \citep{angelopoulos2023prediction}, while leaving the validity of the rest of a generated dataset open.

%% file: sections/limitations.tex
\section{Limitations and implications}\label{sec:limitations}

The experiments show that request construction can distort relationships in two model
families, with consistent request-size evidence in ESS and NHANES. ACS extends the study
to economic data, but all six structured ten-person conditions fail validity and Llama's
effects differ in direction across seeds. The study uses one basic design for profiles
and examples, mainly at temperature $1$; temperature and wording checks broaden this
evidence without establishing invariance across models, prompts or request sizes.
Matching cannot recover rejected answers, including those lost in three original ESS
ten-person conditions. Because the analysis uses complete outcomes and equal group weights, its conclusions
describe selected reference populations rather than survey-weighted national populations.

The history interventions establish earlier answers as a causal channel while leaving
the model's internal computation open. Hiding history also
skips syntax and state updates in a way that may be unfamiliar to the model. Although
the Qwen/NHANES improvement persists after removal of the token-position gap, worse
marginal accuracy and mixed pattern results leave general weakening as a plausible
explanation for much of the lower correlation error. The matched follow-up examines
profile context and history in one runtime using forced continuations; differences in
runtime, precision, cohort and prompt order prevent attribution of separate shares of the
original batching penalty to these factors. Turning the interventions into reliable
generation methods requires further work.

The correction experiments and new diagnostics remain exploratory because they reuse
previously examined human references, even though human evaluation answers are excluded
from fitting. The scalar still requires one-person outputs, while exact marginal
preservation uses the evaluation batch itself and can weaken agreement with profiles.
Neither establishes a generally useful or cheaper correction. Uncertainty estimates also
have defined limits: bootstrap intervals cover only specified sources, and reference
resampling lacks survey-cluster identifiers. The descriptive audit and archival experiments
provide supporting context from different populations and designs
(Appendices~\ref{app:audit_results}--\ref{app:campaign}).

Request construction can change the statistical relationships attributed to a synthetic
population, so the number of records sharing a completion and their access to earlier
answers belong in the generator's specification and validation. Evaluation must then
address the marginal, dependence and regression quantities required by the intended
analysis, since improvement in one does not establish the validity of the others.

%% file: sections/statements.tex
\subsubsection*{Ethics statement}

The study analyzes existing, de-identified survey and public-use microdata. It collects
no new human-subject responses. Some source surveys require registration and restrict redistribution.
Training examples sent to model providers contain real survey rows, and generated records
include sampled conditioning profiles.

\subsubsection*{Reproducibility statement}

The protocols and hashes identify the code and inputs used in each study
(Appendices~\ref{app:request_design}--\ref{app:prefix_intervention}). Independent checks
verify generated tokens, matched respondents, weights, corrections, error measures and
paired bootstrap draws. Reports include invalid outputs and undefined intervals.
Regression coefficients agree with a separate weighted-least-squares implementation using
explicit group indicators (Appendix~\ref{app:downstream}). The supplement contains
data-preparation and experiment code, artificial-data tests, and aggregate results that
reproduce the figures and generated tables offline. A single CPU command runs the replay
and checks its outputs.
Recomputing respondent-level results or rerunning generation requires access to the source
microdata, model weights and compatible GPU runtimes. The included preparation scripts
construct the train/test split and study inputs; fresh runs require new provenance manifests.
The supplement excludes private profiles, fitted maps, real examples and raw requests.

\subsubsection*{AI use statement}

The evaluated models generated synthetic respondents as the experimental treatment. Their
identifiers and settings appear in the methods and appendix. Separately, AI coding assistants,
including OpenAI Codex, supported drafting and revision, literature discovery, methodological
feedback, proof checking and revision, code implementation and debugging, and numerical
analysis and interpretation. The author retains responsibility for all claims, citations,
proofs, code and results.

%% file: sections/app_request_design.tex
\section{Matched request-size experiment}\label{app:request_design}

\subsection{Sampling and study design}

The September 16 protocol draws on earlier experiments and was fixed before these
responses were generated. It specifies two models, request sizes one and ten, two seed policies and
three repetitions. The protocol was not externally preregistered. Both model runs are
complete and independently checked. The separate prompt-context extension is reported
on its own.

The eligible pool consists of ESS training rows complete on nineteen numeric targets
(excluding household size), using the split and registry in Appendix~\ref{app:current}. Countries are ranked by eligible
training count, with alphabetical ties: DE, FR, FI, AT, SE, BE, EE, NL, ES, GB, IE, IT, NO,
CZ, CH, BG, PL, HU, LT and SI. Sampling without replacement selects $100$ rows per country
(seed $20260916$). Their ten categorical fields condition generation; missing fields remain
\emph{not reported}. Their numeric outcomes are withheld from generation and saved only
for a secondary reference.

The independent test reference contains $19{,}896$ complete rows. Each country receives
weight $1/20$, so estimates describe complete cases with equal country weights rather
than survey-weighted populations. The profiles form forty sets of fifty,
each containing one country and sharing three fixed real examples (seed $20260917$), sampled from $79{,}993$
training rows complete on all displayed fields, excluding all $2{,}000$ targets. Examples,
profiles and item order remain fixed across conditions and models.

\subsection{Requests and runtime}

Each set of fifty profiles becomes fifty one-person or five ten-person requests.
The experiment runs under two seed policies. In the primary policy, every request
receives a distinct stable $63$-bit seed derived from its identity. In the diagnostic
policy, every request resets to the same seed for that run. This checks whether reusing
a decoder seed changes the request-size comparison. Seeds $0,1,2$ index complete
repetitions of the experiment under both policies. Each model
requests $24{,}000$ respondents through $13{,}200$ requests across the twelve conditions.

The prompt requests a flat JSON array with the supplied integer IDs and nineteen numeric
fields, permitting fractional values within inclusive registry bounds. A technical pilot showed that nested examples made the required flat output ambiguous.
The \texttt{flat-examples-v2} amendment changes only the formatting, pairing the same
example profiles with flat answer arrays and reserved negative IDs. Profiles, values,
scales, seeds and comparisons stay fixed. Both amended pilots passed validation before
production; the change was chosen for formatting validity rather than scientific results.
The supplied constructor defines this prompt, which differs from the September 6 panel template.

Checkpoints are Qwen3.8-27B \citep{qwen2026qwen38card}, revision
\texttt{1d4bf0f2ff6012fd82039f2fa52739d0dd7c60c0}, and Llama-3.3-70B-Instruct
\citep{meta2024llama33card}, revision \texttt{6f6073b423013f6a7d4d9f39144961bfbfbc386b}.
Both use two RTX 6000 Ada GPUs, tensor parallelism two, Python $3.12.13$,
vLLM $0.23.0$ \citep{kwon2023pagedattention}, PyTorch $2.11.0$ \citep{paszke2019pytorch}
and Transformers $5.9.0$ \citep{wolf2020transformers}. Qwen uses BF16 weights, automatic
KV-cache dtype, language-model-only mode and thinking disabled; Llama uses on-the-fly FP8
weights and FP8 KV cache. Cross-model magnitudes therefore do not isolate architecture.
Temperature and top-$p$ are $1$, top-$k$ is unrestricted and min-$p$ is zero. Output/context
limits are $8{,}192/16{,}384$ tokens; inference microbatches contain $32$ requests, not
$32$ respondents per completion. Generation stops at the closing array bracket and retains it.
Exact chat-template lengths are checked; output budgets are never silently reduced.

\subsection{Validity, estimands and uncertainty}

The generation pipeline saves each completed request atomically, including its raw text, IDs, prompt and input
hashes, settings, runtime, token counts and finish reason. Acceptance requires exactly the requested unique IDs and fields,
finite in-range JSON numbers and no duplicate keys. Any missing/extra row or field, invalid
number or token-limit termination rejects the whole request. Rejected requests are neither altered nor partially retained or retried. A resumed run
executes only requests with no completed record.

Pilots cover the first two fixed profile sets, seed zero and all four size/policy arms.
They are excluded from production despite reused profiles. The frozen threshold is $99\%$
respondent retention and zero truncations per arm. Failed production arms remain reported.

\begin{table}[htbp]
\centering
\caption{Production validity under whole-request acceptance: an invalid ten-person request
loses all ten respondents. Failed-quality arms remain reported.}
\label{tab:request_design_quality}
\scriptsize
\input{generated/request_design_quality}
\end{table}

The analysis scores both all accepted rows and the profiles accepted in all four conditions within
each model and repetition. Matching aligns the observed profiles, but the comparison
still depends on which answers passed validation. Every fixed country
must retain twenty rows and all $171$ pair correlations must be defined. Centered residuals
receive weights $1/(20n_c)$: this pools country covariances, not country correlations.
The primary contrast is size-ten minus size-one MAE under independent seeds; shared-reset
contrasts and the size-by-policy interaction are diagnostic. The interaction changes sign
across repetitions in both models.

Percentile intervals use $500$ paired country resamples (base seed $20260918$), retaining each
sampled country's full generated/reference records and multiplicity. They condition on realized
generation and selected countries, excluding generation and within-country reference uncertainty.
Three-seed ranges are descriptive; the $171$ overlapping pairs are never independent units.

\begin{table}[htbp]
\centering
\caption{Independent-seed request-size results on four-arm common IDs. MAE averages $171$
within-country correlation errors; intervals use the conditional paired-country bootstrap.
$\dagger$ denotes an arm below $99\%$ validity.}
\label{tab:request_design}
\small
\input{generated/request_design_main}
\end{table}

\begin{table}[htbp]
\centering
\caption{Matched-ID sensitivities: $\Delta$ is size-ten minus size-one MAE. Excluding age/year
leaves $136$ pairs. Intervals use the same conditional country bootstrap. A dagger flags any
failed-quality arm in four-arm matching, including the other seed policy.}
\label{tab:request_design_sensitivity}
\scriptsize
\input{generated/request_design_sensitivity}
\end{table}

\subsection{Secondary checks and verification}

Secondary checks exclude age and year, leaving $136$ pairs; examine the three original
police-trust pairs; and use the sampled profiles' observed human answers as another reference.
Further checks fit weighted, country-demeaned regressions of police trust on income, happiness,
democracy satisfaction and age. On matched Qwen IDs
under independent seeds, median democracy coefficients are $0.435/0.593$ for one/ten-person
requests versus $0.339$ in the held-out reference; happiness coefficients are $0.230/0.169$
versus $0.154$. These coefficients illustrate that batching can move different regression estimates in
different directions. They describe associations, not causal effects.

Post hoc, the $2{,}000$ sampled profiles' human outcomes have MAE $0.0177$ against the test
reference, or $0.0178$--$0.0185$ on model/repetition-specific matched IDs. This comparison describes the particular human samples and does not establish a noise
floor or an uncertainty interval. The outcomes were withheld from generation; examining
them afterward does not create a new holdout or decompose the synthetic error.

An independent row-weight implementation reproduces all $48$ endpoints, $8{,}208$ pair
correlations, $192$ regression coefficients and $300$ contrasts, using full country-dummy
weighted least squares rather than the production analyzer. A separate raw-record audit
reparses all $26{,}400$ completions, matches all $89$ invalid requests and reasons, and exactly
reproduces $47{,}713$ accepted-row exports. Country counts and quality flags agree. A reporting
correction propagates failed-quality flags into seed summaries without changing estimates;
original outputs remain archived. These checks do not validate bootstrap coverage, model
generalization or unobserved rejected answers.

\subsection{Matching only the primary independent-seed arms}\label{app:primary_support}

An exploratory sensitivity matches profiles only across the two independent-seed conditions.
This keeps failures in the shared-reset diagnostic from removing otherwise observed answers
from the main comparison. Countries,
variables and human reference remain fixed. Its absolute contrasts are close to those on
the four-arm support. At seeds $0,1,2$, absolute two-arm contrasts are
$0.1026,0.0913,0.0893$ for Qwen and $0.2098,0.2384,0.2326$ for Llama, versus
$0.1031,0.0913,0.0893$ and $0.2128,0.2384,0.2326$ on four-arm support.
Quality failures remain flagged. Independently weighted covariances verify both the sensitivity
and archived four-arm estimates using audited inputs. The check changes which accepted profiles enter the comparison. It neither changes generation
nor recovers rejected answers.

%% file: sections/app_history_channel.tex
\section{Changing and hiding earlier answers}\label{app:history_channel}

Two interventions test the influence of earlier answers. Changing their pairing tests whether
they influence relationships in later responses. Hiding naturally generated answers tests
whether access to them affects accuracy against human data. These are separate questions:
changing a generated relationship need not make it more accurate. The studies follow the
prefix-replay experiment (Appendix~\ref{app:prefix_intervention}). Each protocol,
source snapshot, input plan and analysis rule was fixed before scientific generation,
including for the extension to two models and two domains.

\subsection{Controlled pairing of fixed history values}

\paragraph{Stimuli and allocation.}
NHANES uses BMI/sedentary minutes and eight sex--age templates; ESS uses trust in
parliament/the judiciary and twenty country templates. Each template uses the most common categorical profile in the fitting partition,
with ties broken by a fixed hash. Prompts contain ten identical profiles except
IDs and three fixed training examples projected onto the two outcomes. ESS uses examples from
each country's first fitting set of fifty profiles; NHANES reuses the preceding frozen Qwen
study's exact templates, prompts, examples and paired seeds. Original manifests verify the inputs.

The model receives answers for the first nine respondents and completes the tenth.
For NHANES, the supplied answers draw from two fixed lists: BMI values
$20,22,\ldots,36$ and sedentary minutes $120,180,\ldots,600$. Pairing the lists in
the same rank order gives a positive correlation: respondents with higher BMI also have
more sedentary minutes. Reversing one list gives a negative correlation. A third pairing,
defined by the index permutation $(0,2,6,8,7,5,4,3,1)$, gives zero correlation.
ESS follows the same procedure using $1,2,\ldots,9$ for each trust variable.

The order in which the nine supplied respondents appear also varies. Eighteen
forward/reverse rotations place every rank twice in every history position, so no rank
is consistently nearest the target. Each pairing retains the same values for each variable
and the same number of tokens for the relevant tokenizer. Because the supplied respondents
have identical profiles, moving values between them also preserves agreement with their
attributes. It changes both the relationship between the variables and the sequence of
values that the model encounters.

Each order has three distinct seeds paired across conditions. A fourth, hidden-history condition
skips preceding-answer state updates while retaining their logical token-position gap.
The equal-length histories produce the same input when hidden, so one hidden-history
condition serves as their shared reference.
Each model generates $8\times18\times3\times4=1{,}728$ NHANES and
$20\times18\times3\times4=4{,}320$ ESS continuations. Only newly generated tenth answers are scored.

\paragraph{Endpoint and uncertainty.}
For each template, Pearson correlations use responses accepted under every condition,
requiring at least twenty of the $54$ planned order--repetition cases. The endpoint averages
the positive-minus-negative contrast equally across templates. This averages correlations
rather than covariances, unlike Equation~\ref{eq:within}, and does not use human answers
as an accuracy reference. All-retained, per-repetition, Spearman
and marginal-shift sensitivities remain in the aggregate supplement.

The $5{,}000$ fixed-seed bootstrap draws independently resample eighteen order blocks within
each template, carrying all repetitions and arms together. The intervals describe sensitivity to the mix of orders. They do not provide population
coverage because the balanced orders are not independent sampled contexts and the templates
are not a random demographic sample. Any undefined draw suppresses the interval.
Llama/NHANES has $1{,}393$ such draws; its zero-history-pairing arm also has one undefined
template correlation. Neither is replaced by zero or dropped.

\begin{table}[t]
\centering
\caption{Equal-template correlations in newly generated tenth responses. Hidden history skips
earlier-answer updates; $\Delta_{+,-}$ is positive minus negative pairing. Intervals describe
conditional order-mixture sensitivity. A dash is undefined; Llama/NHANES's interval is withheld.}
\label{tab:history_controlled}
\footnotesize
\setlength{\tabcolsep}{3pt}
\input{generated/history_controlled}
\end{table}

\begin{table}[t]
\centering
\caption{Controlled-history validity and matched support under one strict rule across arms.
Invalid answers remain archived; undefined bootstrap draws are counted, not discarded.}
\label{tab:history_validity}
\scriptsize
\setlength{\tabcolsep}{3pt}
\input{generated/history_validity}
\end{table}

\paragraph{Exploratory within-order check.}
An apparent pairing effect could arise because different answer orders shift the mean
of the generated responses. A check for this explanation holds the profile template,
answer order and pairing condition fixed, then subtracts the mean across retained
repetitions from each outcome.
Correlations calculated from these centered responses remain positive under positive
pairing and negative under negative pairing. The respective correlations are $0.933/-0.296$ (Qwen/NHANES), $0.978/-0.480$ (Qwen/ESS),
$0.649/-0.225$ (Llama/NHANES) and $0.634/-0.068$ (Llama/ESS).
Differences between order-specific means therefore explain only part of the contrast.
This exploratory check has three repetitions per order, requires at least two retained,
and supplies no formal interval. It leaves the original endpoint unchanged. Both this check
and the balanced history positions leave open whether the model copies covariance or
extends the sequence of each variable separately.

\subsection{Fresh natural histories at three positions}

\paragraph{Cohort and intervention.}
Qwen/NHANES uses all $400$ evaluation profiles, fifty per sex--age cell, and five outcomes:
poverty--income ratio, BMI, self-rated health, PHQ-9 and sedentary minutes. A fixed hash order selects nine distinct fitting donors from the same cell for each
target and repetition. Full prompts include
these donors, the disjoint target and its existing three examples; target outcomes never
select donors or construct prompts.

At position ten, the model first generates answers for the nine preceding respondents.
The actual target profile is already present in the prompt while it does so. The model
then generates the target's answer twice, starting each time from a reconstructed state.
In the visible condition, the model processes the earlier answers as usual. The hidden
condition skips those tokens, including their answer syntax, so they cannot update
attention keys and values or
convolution/recurrent states. The two continuations use the same full profile prompt,
forced target syntax, decoder seed and absolute logical target positions. Keeping those
positions fixed leaves a gap where the earlier answers would have been, so information
about their length remains available.

The extension moves the target to position two or five, preserving donor order, examples and
target seed. It generates new one- or four-row histories under the changed full prompt, rather
than truncating a nine-row history. Each position has three repetitions, $1{,}200$ planned
histories and $2{,}400$ visible/hidden continuations. For position ten, the analysis reuses the verified original outputs and rescores them on
profiles accepted in all six position--visibility conditions. At position two, one preceding answer cannot itself
define a cross-respondent covariance. Moving the target changes history length, prompt order and predecessor profiles together.
The comparison therefore cannot isolate the effect of history length alone.

\paragraph{Scoring and uncertainty.}
Matching profiles across all six position--visibility conditions leaves $44$--$50$ targets
per cell in each repetition. Invalid histories invalidate both associated targets; invalid targets remain archived,
without retries or clipping. Equal-cell pooled covariances give correlations for all ten pairs,
scored against the study's fixed human reference. The primary contrast averages hidden-minus-visible MAE across repetitions. A negative
value means lower error, without implying that the resulting data are accurate.

For $5{,}000$ paired bootstrap draws, target IDs are resampled from the fifty planned profiles
within each cell, carrying positions, arms and repetitions together with multiplicity; each
repetition's six-arm acceptance rule is then applied. The reference stays fixed. These
profile-block intervals condition on the observed cohort and intervention, not new model-training
runs, reference populations or generation seeds. All three contrasts have no undefined draws.
Aggregate reports also retain all-retained, cell, focal-pair, repetition and failure diagnostics.

\begin{table}[t]
\centering
\caption{Qwen/NHANES, six-arm common support: three-repetition means of correlation MAE $A$
and normalized marginal Wasserstein error $M$. Change is hidden minus visible; intervals
resample paired profile blocks. Valid-target counts precede six-arm matching; each position
plans $1{,}200$ histories.}
\label{tab:history_positions}
\scriptsize
\setlength{\tabcolsep}{2.5pt}
\input{generated/history_positions}
\end{table}

Valid histories number $1{,}199$, $1{,}195$ and $1{,}192$ at positions two, five and ten.
The original position-ten two-arm estimate falls from $0.363$ to $0.194$
($\RelOriginalHistoryGain\%$): contrast $-0.170$, interval $[-0.221,-0.125]$.
Table~\ref{tab:history_positions} instead gives the extension's six-arm comparison; neither
replaces the other's prespecified estimate.

\subsection{Accuracy extensions: Llama/NHANES and Qwen/ACS}\label{app:history_extensions}

Each tenth-respondent extension follows a separately fixed protocol and uses $400$ existing evaluation profiles
(fifty per sex--age group), fitting-partition donors and the existing human reference:
$1{,}200$ fresh histories and $2{,}400$ paired targets across three repetitions.
Llama/NHANES retains its five-outcome plan and seeds; Qwen/ACS scores all six pairs among
income, wages, weekly hours and weeks worked. Both interventions preserve the full prompt,
target syntax, seed and logical positions. These are reused cohorts, not new holdouts.

The endpoint averages hidden-minus-visible MAE across repetitions on two-arm common support.
Independent audits accept every history and target without truncation, so matching retains all
fifty profiles per group/repetition. The $5{,}000$ paired profile-block resamples carry repetitions
and arms together with the reference fixed; all draws are defined. These are conditional
intervals, not population or generation-seed uncertainty.

\begin{table}[t]
\centering
\caption{Tenth-respondent extensions: three-repetition mean correlation MAE $A$ and normalized
marginal Wasserstein error $M$. Reduction is relative to visible history; $\Delta A$ is hidden
minus visible. All $400$ profiles per repetition are retained; intervals use paired profile blocks.}
\label{tab:history_extensions}
\scriptsize
\setlength{\tabcolsep}{3pt}
\input{generated/history_extensions}
\end{table}

Although mean marginal error rises (Section~\ref{sec:history_channel}), ACS's second
repetition improves it by $0.09\%$. Prespecified ACS sensitivities
give correlation-error contrasts $-0.085$ without the income--wages pair and $-0.080$ after
transforming both monetary variables by $\operatorname{asinh}(x/10{,}000)$; each repetition
improves. Neither establishes that reducing correlation is generally beneficial.

\subsection{Matched singleton, profile-context and position controls}\label{app:context_controls}

The follow-up reuses the $400$ Qwen/NHANES target profiles, three repetitions, donor profiles,
examples, original native histories and target seeds in the same pinned FP16 Transformers
runtime. The model generates fresh answers under five conditions. The first presents the target
profile alone. The second puts the target first in a prompt containing all ten profiles,
so the model sees the other profiles but has no preceding answers. The third presents
the original visible answer history. The final two conditions both hide that history:
one retains the gap in token positions, while the other closes the gap and assigns
contiguous positions to the tokens the model actually processes.
All five use forced target-row syntax and the same bounded numeric grammar. The one-person condition uses the same continuation procedure as the other four. It differs
from the original vLLM experiment, which generates complete arrays. Comparisons of first
and later positions also change profile order in the prompt.

Eight invalid original histories exclude all five associated continuations. The $1{,}192$
eligible cases yield $5{,}960$ new outputs, retaining the full $6{,}000$-row planned disposition.
Five-arm common support contains $394$, $395$ and $394$ profiles across the three repetitions.
There are $1{,}190$ valid singleton targets, $1{,}190$ target-first outputs, $1{,}192$ visible
outputs, and $1{,}183$ in each hidden arm before matching. Invalid outputs are retained in
the audit and never replaced. Settings and the analysis contrasts were fixed before this
generation, after the earlier findings were known; this is an exploratory reused cohort.

\begin{table}[htbp]
\centering
\caption{Same-runtime Qwen/NHANES continuations, three-repetition means on five-arm common
support. $P$ is correlation-vector Pearson agreement, $A_{\rm or}$ human-oracle scalar MAE,
and $E_\beta$ the fixed PHQ-9 coefficient error. All five arms use forced target syntax.}
\label{tab:review_context}
\small\setlength{\tabcolsep}{4pt}
\input{generated/review_context}
\end{table}

\begin{table}[htbp]
\centering
\caption{Prespecified contrasts within the exploratory context follow-up. Paired bootstrap
draws resample target IDs within cells, carrying all arms and repetitions. The second interval
additionally resamples human reference rows within cells; neither is survey-design uncertainty.
All $2{,}000$ draws are defined in each sensitivity.}
\label{tab:review_context_contrasts}
\small\setlength{\tabcolsep}{3pt}
\input{generated/review_context_contrasts}
\end{table}

Both ways of hiding history reduce raw MAE, weaken correlations and improve pattern
agreement relative to visible history. With contiguous positions, marginal error still rises:
$M_{\rm SD}=0.446$ versus $0.402$ with visible history. A position gap alone therefore does
not account for the mean gain in this setting. The target-first profile context raises raw
MAE relative to singleton generation, despite slightly better pattern agreement and oracle
MAE. Profile context can therefore change relationship strength and pattern differently. These
comparisons cannot assign separate causal shares of the original batching effect to context
and history.

Checks before generation verified the original runtime and source bindings. Twenty fictional
cases compared individual and padded execution of all five routes through three forced
decoding steps, using the existing maximum-logit and total-variation limits. An independent CPU audit rechecks
native tokens, grammar, routes, positions and hashes, and reproduces correlations with explicit
weighted NumPy covariance. It verifies saved GPU gate evidence rather than independently
rerunning the GPU computations. All $1{,}192$ recovered raw case files match their audit hashes.

\subsection{Runtimes, validation and independent audits}

Original Qwen/NHANES history/position studies use the pinned checkpoint in FP16, Transformers
$5.9.0$, PyTorch $2.11.0$, xgrammar $0.2.2$ \citep{dong2025xgrammar}, stock hybrid-state SDPA
and two RTX6000 Ada GPUs. The crossed pairing study uses Qwen FP16 and Llama NF4 double
quantization with FP16 computation (bitsandbytes $0.50.0$ \citep{bitsandbytes2026}) on one GB10,
with Transformers $5.15.0$, PyTorch $2.13.0$ and xgrammar $0.2.3$. Missing offline FP8 kernels
motivated NF4 before generation; the permitted smaller-model fallback was unused.
Cross-model magnitudes do not isolate architecture.

The crossed Qwen/NHANES study repeats the same stimuli and seeds to check agreement
between runtimes, rather than adding an independent experimental setting: its contrast is $1.633337$ originally and $1.633224$ on GB10. Runs are not
pooled as extra seeds. Original studies use temperature $1$, unrestricted top-$p$/top-$k$,
independent CPU random streams paired across conditions, and a nonnegative numeric grammar
allowing four integer/two decimal digits; operational ranges are checked afterward.
Caps are $128$ target tokens for controlled pairs, $192$ for natural five-outcome answers,
and $1{,}024$ for histories. Completed invalid/truncated outputs are never selectively rerun.

The accuracy extensions use two RTX6000 Ada GPUs: Llama NF4/FP16 SDPA and Qwen FP16 SDPA,
distinct from Llama's original FP8 and controlled-history GB10 runtimes.
Following failed fictional validation, ACS adopts a variable-specific range grammar allowing
signed values as required, seven integer/six decimal digits, source microbatches of four,
singleton target execution, and $2{,}048/256$ source/target caps. Singleton execution avoids
failed padded-target numerical checks; Llama/NHANES retains $1{,}024/192$ caps.
Failed attempts and source snapshots remain archived.

Before scientific generation, numerical checks compare one-shot and split replay of native
tokens, individual and batched execution with visible history, and multiple decoding steps
with gapped hidden history. They test finite caches, prompt sensitivity, grammar and ranges. Qwen checks sixteen KV and forty-eight convolution/recurrent pairs; Llama
checks eighty KV pairs. Duplicate hidden routes match states, logits and common-seed outputs
exactly; an empty-history sham is also technical, not scientific, evidence. Admission thresholds
remain maximum logit difference $0.5$ and total variation $0.02$; crossed Qwen/Llama gates
attain $0.0723/0.00435$ and $0.1094/0.00696$. Amended ACS source/target checks pass the same limits.

Each run records its checkpoint inventory, tokenizer, input and source hashes, packages,
devices and precision. The position studies also match the original runtime identity.
Outputs are saved atomically before validity decisions. Independent CPU audits reconstruct native outputs, validity, routes, positions,
seeds and hash bindings, then recompute endpoints and bootstrap draws without trusting derived
validity flags. Extension statistics agree within $10^{-12}$; ACS additionally replays every token
through the production grammar checker. These audits do not independently rerun GPU arrays,
rehash physical model weights or rebuild extension demographic plans.
The anonymous supplement includes audit code, artificial tests and aggregate replay artifacts,
not real examples, profiles or individual generated records.

%% file: generated/history_extensions.tex
\begin{tabular}{llrrrrrr}
\toprule
Model & Domain & $A_{\rm visible}$ & $A_{\rm hidden}$ & Reduction & 95\% interval for $\Delta A$ & $M_{\rm visible}$ & $M_{\rm hidden}$ \\
\midrule
Qwen & ACS & 0.166 & 0.103 & 38\% & $[-0.117, -0.015]$ & 0.0390 & 0.0399 \\
Llama-70B (NF4) & NHANES & 0.577 & 0.300 & 48\% & $[-0.306, -0.243]$ & 0.0767 & 0.1060 \\
\bottomrule
\end{tabular}

%% file: sections/app_downstream.tex
\section{Exploratory downstream regression checks}\label{app:downstream}

These checks test whether improvements in correlation or marginal accuracy translate into better
regression estimates. The specifications were fixed after inspecting those earlier results
and before computing the regressions. They reuse the same outputs, cohorts and human
references, so the checks remain exploratory. The regressions describe associations and
do not provide causal or clinical validation.

\paragraph{Fixed specifications and common scale.}
ESS regresses trust in police on income position, age, happiness and satisfaction
with democracy. NHANES regresses PHQ-9 depression score on sedentary minutes, BMI,
poverty--income ratio and self-rated health. ACS regresses wage/salary income on
weekly hours and weeks worked, retaining zero wages and excluding total income,
which contains wage income. Every fit retains the original scales and observations and uses all specified predictors.

Each fit includes group fixed effects and observation weights $1/(|C|n_c)$,
giving each country or demographic cell equal weight. Comparing coefficients across
conditions requires a common scale. Pooled within-group standard deviations from the
human reference provide that scale: the same values standardize every human and
synthetic coefficient. A method cannot therefore improve this measure simply
by changing the standard deviations used for scoring. For the selected $k$ predictors,
the error measure averages absolute differences from the human coefficients:
\[
 E_{\beta}=\frac{1}{k}\sum_{j=1}^{k}
 \left|\frac{s_{H,j}}{s_{H,y}}\bigl(\widehat\beta_{S,j}-\widehat\beta_{H,j}\bigr)\right|.
\]
Here $\widehat\beta$ denotes native-unit slopes; human scales $s_H$ and reference slopes
are fixed. Each fit requires twenty finite complete observations per original group,
positive human scales, and a predictor-covariance condition number no larger than $10^{10}$
on the common standardized scale. One unsupported repetition suppresses its three-repetition
summary; no pseudoinverse or reduced specification is substituted. All $162$ planned fits
are supported and agree with independent explicit-dummy weighted least squares. Aggregate
outputs retain native/standardized slopes, $552$ coefficient comparisons, repetitions and
quality flags. No population-confidence interval or coefficient-wise significance is claimed.

\paragraph{Natural histories.}
Qwen/NHANES uses the six-arm common cohort from Appendix~\ref{app:history_channel}.
Hiding history lowers the average coefficient error at every combination of target
position and repetition, giving nine improvements. That average can hide changes in
individual coefficients: $12/36$ coefficient--repetition errors worsen. In particular, mean BMI-coefficient error rises by $\RelHistoryBMITwo\%$ at position two
and $\RelHistoryBMIFive\%$ at position five. Moving the target changes history length, order and preceding profiles together, so these
comparisons cannot isolate the effect of history length.

\begin{table}[t]
\centering
\caption{Exploratory PHQ-9 regression: three-repetition mean standardized-coefficient
MAE $E_{\beta}$ on the Qwen/NHANES six-arm common cohort. Hiding histories lowers error
by $25\%$, $30\%$ and $35\%$ at positions two, five and ten, respectively.
``Improved'' refers to the mean over four coefficients, not each coefficient.}
\label{tab:downstream_histories}
\footnotesize
\input{generated/downstream_histories}
\end{table}

\paragraph{Regression accuracy after correction.}
Relative to untouched batching, the learned map reduces mean coefficient error by
$\RelDownstreamESSQwenLearned\%$ for Qwen/ESS, $\RelDownstreamESSLlamaLearned\%$ for
Llama/ESS and $\RelDownstreamNHANESQwenLearned\%$ for Qwen/NHANES; Llama/NHANES changes
little ($\RelDownstreamNHANESLlamaLearned\%$ reduction). ACS error rises by
$\RelDownstreamACSQwenLearned\%$ for Qwen and $\RelDownstreamACSLlamaLearned\%$ for
Llama, from a small Qwen baseline of $0.0315$. Overall, $27/60$ coefficient--repetition
and $6/18$ repetition-level mean errors worsen. Exact preservation lowers Llama/ACS error
by $\RelDownstreamACSLlamaPreserved\%$ but raises Qwen/ACS and Llama/NHANES errors by
$\RelDownstreamACSQwenPreserved\%$ and $\RelDownstreamNHANESLlamaPreserved\%$.
In Qwen/ACS, uncorrected batching outperforms the other seven original methods.
Table~\ref{tab:review_downstream} separately reports the later scalar controls. The scalar
that preserves marginals exactly improves this setting, so the original ranking does not
extend to every subsequent variant.
Lower average correlation error therefore does not guarantee better regression coefficients.

\begin{table}[t]
\centering
\caption{All correction variants: three-repetition mean standardized-coefficient
MAE $E_{\beta}$ for the fixed domain-specific regressions. Lower is better; all
variants share the original matched evaluation IDs within each model and repetition.
$\dagger$: the six original ACS size-ten generation-quality failures remain included.}
\label{tab:downstream_corrections}
\footnotesize
\setlength{\tabcolsep}{3pt}
\input{generated/downstream_corrections}
\end{table}

%% file: generated/downstream_histories.tex
\begin{tabular}{rrrrr}
\toprule
Position & Visible & Hidden & Hidden minus visible & Repetitions improved \\
\midrule
2 & 0.105 & 0.079 & -0.026 & 3/3 \\
5 & 0.134 & 0.093 & -0.040 & 3/3 \\
10 & 0.156 & 0.101 & -0.055 & 3/3 \\
\bottomrule
\end{tabular}

%% file: sections/app_correction.tex
\section{Dependence-map construction and diagnostic comparisons}\label{app:correction}

The persistence of exaggerated relationship strength raises a question about correction:
how much does a full set of pair-specific adjustments contribute beyond weakening all
relationships together? A comparison between a full matrix of targets and a single overall
adjustment addresses this question using the same one-person fitting outputs. Both rely
on standard normal-score and whitening/recoloring operations, with marginal-preserving
rank adjustment also established \citep{iman1982distribution}. Their value here is
diagnostic: removing pair-specific target information reveals how much of the apparent
gain depends on it. These checks reuse previously examined outputs and human references,
so the results remain exploratory.

\paragraph{Pair-specific and single-parameter corrections.}
For each domain, model and seed, whole requests are divided into fitting and evaluation
sets, with profiles matched to the corresponding one-person outputs. Each set plans fifty
profiles per group ($1{,}000$ in ESS or $400$ in NHANES/ACS), and fitting uses only
answers accepted under both request sizes. Human answers provide the evaluation reference
but do not select the corrections' parameters.

The correction first puts each variable on a common scale by ranking fitting answers
within each group and mapping their percentiles to a standard normal distribution.
These transformed values, or normal scores, provide the basis for estimating correlations;
a fixed random ordering keyed to respondent IDs resolves ties. Let $R_B,R_T$ denote the
correlation matrices for batched and one-person fitting scores, pooling group covariances
with equal weights. The full correction uses these matrices to move evaluation scores
$z$ toward the one-person dependence structure:
\begin{equation}\label{eq:correction_map}
z^*=(z-\mu_{B,c})D_B^{-1}\widetilde R_B^{-1/2}\widetilde R_T^{1/2},
\qquad \widetilde R=0.9R+0.1I,
\end{equation}
where $\mu_{B,c}$ and $D_B$ are fitting group means and pooled standard deviations, and
matrix powers are symmetric. Returning to answer units involves finding each corrected
score's percentile among the corrected fitting scores, then assigning the value at that
percentile in the one-person fitting answers. Although this step draws on the one-person
reference, the resulting evaluation distributions and correlations need not match it exactly.

The scalar correction follows the same procedure but multiplies every batched off-diagonal
correlation by one number, leaving the diagonal at one. It replaces $R_T$ with
$I+\widehat\alpha(R_B-I)$, where
\begin{equation}\label{eq:scalar_fit}
\widehat\alpha=\arg\min_{0\leq\alpha\leq1}\sum_{j<k}|\alpha R_{B,jk}-R_{T,jk}|.
\end{equation}
One multiplier therefore applies to all pairs within a model, domain and seed. Keeping
all other operations and fitting data fixed makes the full map's advantage a test of
the additional value of pair-specific information.

\begin{table}[t]
\centering
\caption{\textbf{A single overall adjustment retains most of the full correction's improvement.}
Three-seed mean correlation MAE on matched evaluation respondents. $A_0$ is the error
from predicting zero correlation; one-person outputs provide another comparator. Gain
retained is $(A_B-A_{\rm scalar})/(A_B-A_{\rm full})$. Both corrections use one-person
fitting marginals and exclude human evaluation answers from fitting. $\dagger$: all six
ACS ten-person quality failures are included.}
\label{tab:review_scalar}
\small\setlength{\tabcolsep}{4pt}
\input{generated/review_scalar}
\end{table}

\paragraph{The value of pair-specific information.}
The scalar retains 93--101\% of the original map's MAE improvement, leaving a mean
advantage of only $0.004$--$0.013$ for the full map in five settings; in Qwen/NHANES,
the scalar performs slightly better (Table~\ref{tab:review_scalar}). Most of the gain
thus survives without a separate target for each pair, although both methods still use
one-person data to set marginal distributions and overall correlation strength.

The baseline comparisons further limit what lower MAE establishes. Both full corrections
in NHANES have larger errors than predicting zero correlations, while setting the target
to an identity matrix gives lower MAE still ($0.149$ for Qwen and $0.135$ for Llama)
despite much weaker pattern agreement. This identity target asks for zero correlations
on the normal-score scale; conversion back to answers can change those correlations,
so the final dataset need not have zero Pearson correlations.
Appendix~\ref{app:review_metrics} reports pattern, sign and zero-baseline skill for every
principal condition.

\paragraph{What the transformation changes.}
Examining the fitting data at successive stages helps explain why the full correction
can outperform the one-person outputs it uses as a reference. Immediately after the
linear adjustment, normal-score correlations are too strong on average in all six
comparisons, but conversion back to answer units leaves them too weak relative to the
one-person fitting answers. The later stages therefore change relationship strength in
a direction that can improve MAE without recovering the human pattern. Removing the
$0.1$ regularization also improves MAE in five settings, so that regularization alone
cannot explain the gains. Table~\ref{tab:review_sensitivity} examines the further roles
of midpoint tie treatment and conversion to normal scores and back without a linear adjustment.

\paragraph{Preserving distributions and testing usefulness.}
Because the correction changes individual-variable distributions as well as relationships,
a further comparison holds those distributions fixed. For each group and variable,
it keeps the original list of batched values and assigns them to people in the order of
their corrected scores. The values themselves stay the same; what changes is which person
receives each value. This exploratory Iman--Conover-style comparison uses the evaluation
batch itself, preserving every marginal distribution and variance. Changes in those
distributions, including their influence on pooled correlations, can therefore no longer
explain a change in accuracy, although reordering can still weaken relationships or
change agreement with profiles.

Under exact preservation, the full map has MAE $0.179$ for Qwen and $0.186$ for Llama
in ESS, with corresponding values of $0.247$ and $0.303$ in NHANES and $0.160$ and
$0.210$ in ACS. These averages conceal one worsening ACS/Llama seed. The scalar controls
also undergo evaluation with exact preservation, without selection based on human-reference
performance (Appendix~\ref{app:review_scalar}).

The practical importance of preserving distributions is apparent in the original
correction, which worsens individual-variable accuracy in every setting. In ACS,
SD-normalized error rises from $0.376$ to $0.423$ for Qwen and from $0.418$ to $0.525$
for Llama. Nor does lower correlation error guarantee better regression estimates:
ACS/Qwen coefficient MAE rises from $0.032$ to $0.124$ under the full map and to
$0.110$ under the scalar. Other settings, including ESS and Qwen/NHANES, show improvements
(Table~\ref{tab:review_downstream}), making the intended analysis essential to deciding
whether a correction is useful.

\paragraph{Reference size and practical value.}
The study plans as many one-person fitting rows as batched evaluation rows, making direct
use of the one-person outputs a practical alternative. Reducing the one-person reference
from fifty to as few as five rows per group, while keeping all batched fitting rows,
often lowers MAE further (Figure~\ref{fig:review_reference_size}). This pattern is consistent
with additional weakening of correlations rather than a benefit from more information,
and does not determine a recommended reference size. Any claim of a cost or throughput
advantage over parallel one-person requests would also require measurements of prefix
reuse, batching efficiency and decoding costs in the intended serving system.

\paragraph{Marginal-error scales.}
The distributional comparisons use two normalizations of Wasserstein distance. For $p$
outcomes and groups $C$, these are
\begin{equation}\label{eq:marginal_error}
 M=\frac{1}{|C|p}\sum_{c,j}\frac{W_1(\widehat F^{\rm synth}_{cj},\widehat F^{\rm real}_{cj})}{u_j-\ell_j},
 \qquad
 M_{\rm SD}=\frac{1}{|C|p}\sum_{c,j}\frac{W_1(\widehat F^{\rm synth}_{cj},\widehat F^{\rm real}_{cj})}{s^{\rm real}_{cj}}.
\end{equation}
Here $W_1$ is the empirical Wasserstein distance, $[\ell_j,u_j]$ the fixed operational
range, and $s^{\rm real}_{cj}$ the human-reference standard deviation. The range scale
is reported for all comparisons, with greater emphasis on $M_{\rm SD}$ for ACS;
preserving each marginal exactly also preserves both error measures.

%% file: sections/app_review_diagnostics.tex
\section{Relationship strength, pattern and correction diagnostics}\label{app:review_diagnostics}

These analyses examine how much of an apparent improvement comes from weakening
correlations and how much reflects a better pattern of relationships. Their design followed
review of the original results. They include all six model--domain settings and all three
seeds, including quality failures. New fits use only the original generated fitting data.
Transformed rows were saved before scoring against human references.
This keeps evaluation answers out of parameter fitting, although the previously examined
references do not provide a new confirmatory test. The supplement includes aggregate
tables, source hashes and analysis code.

\subsection{Strength, pattern and the zero-correlation baseline}\label{app:review_metrics}

These diagnostics use the within-group relationship measure in Equation~\ref{eq:within}.
The calculation gives groups equal weight, averaging their within-group covariance
matrices before converting to correlations. This gives one pooled within-group correlation for each variable
pair. The controlled history-pairing experiment uses a separate endpoint without a human
reference, defined in Appendix~\ref{app:history_channel}.

A vector $x$ collects the synthetic correlations, and a vector $y$ collects the
corresponding human correlations, with variable pairs in the same order. With $p$ variables,
each vector has $K=p(p-1)/2$ entries. All entries use the original equal-group
pooled-covariance scale. Two benchmarks put the observed error in context: predicting
zero for every pair, and choosing the best single multiplier for all synthetic correlations.
The diagnostics are
\begin{align*}
 A&=K^{-1}\sum_k|x_k-y_k|,\qquad A_0=K^{-1}\sum_k|y_k|,\qquad S_0=1-A/A_0,\\
 A_{\rm or}&=\min_{\alpha\in[0,1]}K^{-1}\sum_k|\alpha x_k-y_k|.
\end{align*}
Mean signed error is $K^{-1}\sum_k(x_k-y_k)$. Because positive and negative errors
can cancel, a second measure captures magnitude bias as the average of $|x_k|-|y_k|$. Tables~\ref{tab:review_metrics_ess_frontier}--\ref{tab:review_metrics_acs} therefore
show mean $|x|$ beside the human value $A_0$. The aggregate files additionally retain both
bias measures and the fraction of inflated pairs for every scored condition.

$P$ and $\rho$ are Pearson and Spearman correlations between $x$ and $y$. They ask
whether pairs with stronger positive relationships in the human data also tend to have
stronger positive relationships in the synthetic data, and likewise for negative ones.
Spearman uses the ranking of the signed correlations; Pearson uses their numerical values.
Both describe agreement across variable pairs, rather than accuracy for individual
respondents. Sign agreement is the fraction
with $\operatorname{sign}(x_k)=\operatorname{sign}(y_k)$; an additional aggregate sensitivity
restricts to $|y_k|>0.05$. Pairs share variables and are not independent replications. ACS has
only six pairs, making pattern summaries particularly sensitive to individual relationships.
The original rank and sign conventions are retained across methods.

The zero predictor assigns zero to every correlation. It supplies a baseline for the
estimates, without generating a dataset or setting a lower bound on error. An identity
target instead applies a regularized correction on the normal-score scale and converts
back to answers, whose Pearson correlations need not be zero. The oracle $A_{\rm or}$
uses the human evaluation reference to choose its multiplier and serves only as a diagnostic. A scalar optimum
is the clipped weighted median of $y_k/x_k$ with weights $|x_k|$ over nonzero $x_k$.
Zero source entries contribute constant loss. This rule exactly minimizes the constrained
absolute-error objective and is independently checked against linear programming for the
eighteen generated-reference latent fits.

\begin{table}[htbp]
\centering
\caption{ESS calibration cohort: three-seed means, $A_0=0.193$. $P/\rho$: Pearson/Spearman
correlation-vector agreement; Sign: sign agreement; $A_{\rm or}$: human-oracle scalar MAE.
Exact variants preserve evaluated batch marginals; answer-fit scalar selects its parameter
using generated fitting answers. $M_{\rm SD}$ is defined in Equation~\ref{eq:marginal_error}.}
\label{tab:review_metrics_ess_frontier}
\scriptsize\setlength{\tabcolsep}{3pt}
\input{generated/review_metrics_ess_frontier}
\end{table}

\begin{table}[htbp]
\centering
\caption{NHANES calibration cohort: same diagnostics as
Table~\ref{tab:review_metrics_ess_frontier}, with $A_0=0.160$. The full and scalar maps improve
raw MAE while remaining worse than the constant-zero correlation predictor.}
\label{tab:review_metrics_nhanes}
\scriptsize\setlength{\tabcolsep}{3pt}
\input{generated/review_metrics_nhanes}
\end{table}

\begin{table}[htbp]
\centering
\caption{ACS calibration cohort: same diagnostics, $A_0=0.356$. All six original ten-person
quality failures remain included. SD-normalized marginal error is foregrounded because
the range scale downweights errors on the broad monetary ranges. There are only six pairs.}
\label{tab:review_metrics_acs}
\scriptsize\setlength{\tabcolsep}{3pt}
\input{generated/review_metrics_acs}
\end{table}

\paragraph{Fisher-scale sensitivity.}
The primary error measure uses differences on the Pearson-correlation scale. An
exploratory check instead averages $|\operatorname{atanh}(x_k)-\operatorname{atanh}(y_k)|$, without
refitting. It places greater weight on differences near perfect correlation. No values are
clipped: a nonfinite transform would invalidate the entire endpoint. All displayed endpoints
are finite. Batching still increases error in the primary ESS study, while full-map versus
scalar rankings change in ACS (Table~\ref{tab:review_fisher}). Changing the error scale therefore does not remove the need to evaluate the intended
analysis directly.

\begin{table}[htbp]
\centering
\caption{Fisher-$z$ correlation MAE, three-seed means. The first two rows use the original
ESS request experiment; subsequent rows use the disjoint fitting/evaluation calibration
study. No parameter is chosen using this sensitivity. ACS retains its quality failures.}
\label{tab:review_fisher}
\small\setlength{\tabcolsep}{5pt}
\input{generated/review_fisher}
\end{table}

\subsection{Single-parameter corrections, fixed marginals and reference size}\label{app:review_scalar}

The primary scalar replaces only the target correlation matrix in the original map with
$I+\widehat\alpha(R_B-I)$, using Equation~\ref{eq:scalar_fit}. Source/target regularization,
normal-score conventions, group fitting means/scales and singleton inverse CDFs remain fixed.
Both methods therefore use the same fitting data and assign marginal values in the same
way. The comparison tests the added value of learning each pair's correlation from
one-person outputs. It leaves open how to choose a multiplier using batched outputs alone.

A second scalar control chooses $\alpha$ from $0,0.05,\ldots,1$ by minimizing the transformed
\emph{fitting answers'} Pearson-correlation MAE against singleton \emph{fitting answers}.
Selection is done separately for singleton and exact-batch marginal assignments. The supplement reports every grid score. Selection does not use the human reference or held-out
one-person answers to choose a parameter, marginal assignment or reported setting. Direct answer-scale correlation-estimate
calibration is additionally exported as an estimate-only diagnostic, with no synthetic-dataset
or regression claim.

For exact marginal preservation, each correction reorders the same original values within
each group and column. This holds group variances fixed. Otherwise, changes in those variances
could change pooled correlations even if each group's correlations stayed the same. Reordering can still attenuate Pearson
correlations, change their pattern and move values between profiles. Exact preservation therefore removes marginal replacement as an explanation but does not
remove all possible weakening of correlations. The tables report both scalar choices and
the full map, without selecting a winner using the human reference.

To test how many one-person outputs the correction needs, a further comparison reduces
the number available for fitting while keeping all batched fitting rows. Within each
group, a fixed hash orders the one-person rows. Each fit uses the first five, ten, twenty
or thirty-five, so each larger subset contains the smaller ones. Repeating this with
three orders per generation seed gives nine results at each reduced budget. The full original fit has one result per seed,
or three in total. The five/ten-row variants fall below the original twenty-row support rule
and are explicitly exploratory; they require five rows per group and finite positive latent
variances. All reported fits are supported. The full ESS reference has $990$--$1{,}000$
paired fitting rows depending on retained requests; the other domains similarly retain their
original acceptance losses. The figure uses realized mean budgets rather than assuming
every planned row survived.

Smaller references often improve MAE. For example, the ESS/Qwen full map has MAE $0.126$
with one hundred singleton rows, compared with $0.147$ at the full budget. This does not establish that fewer responses estimate relationships more accurately.
Coarser empirical distributions and noisier normal-score estimates can change how strongly
the correction weakens correlations. The curves serve as diagnostics; they do not select
a budget or establish an optimum for cost or analytic usefulness.

\begin{figure}[htbp]
\centering
\includegraphics[width=\textwidth]{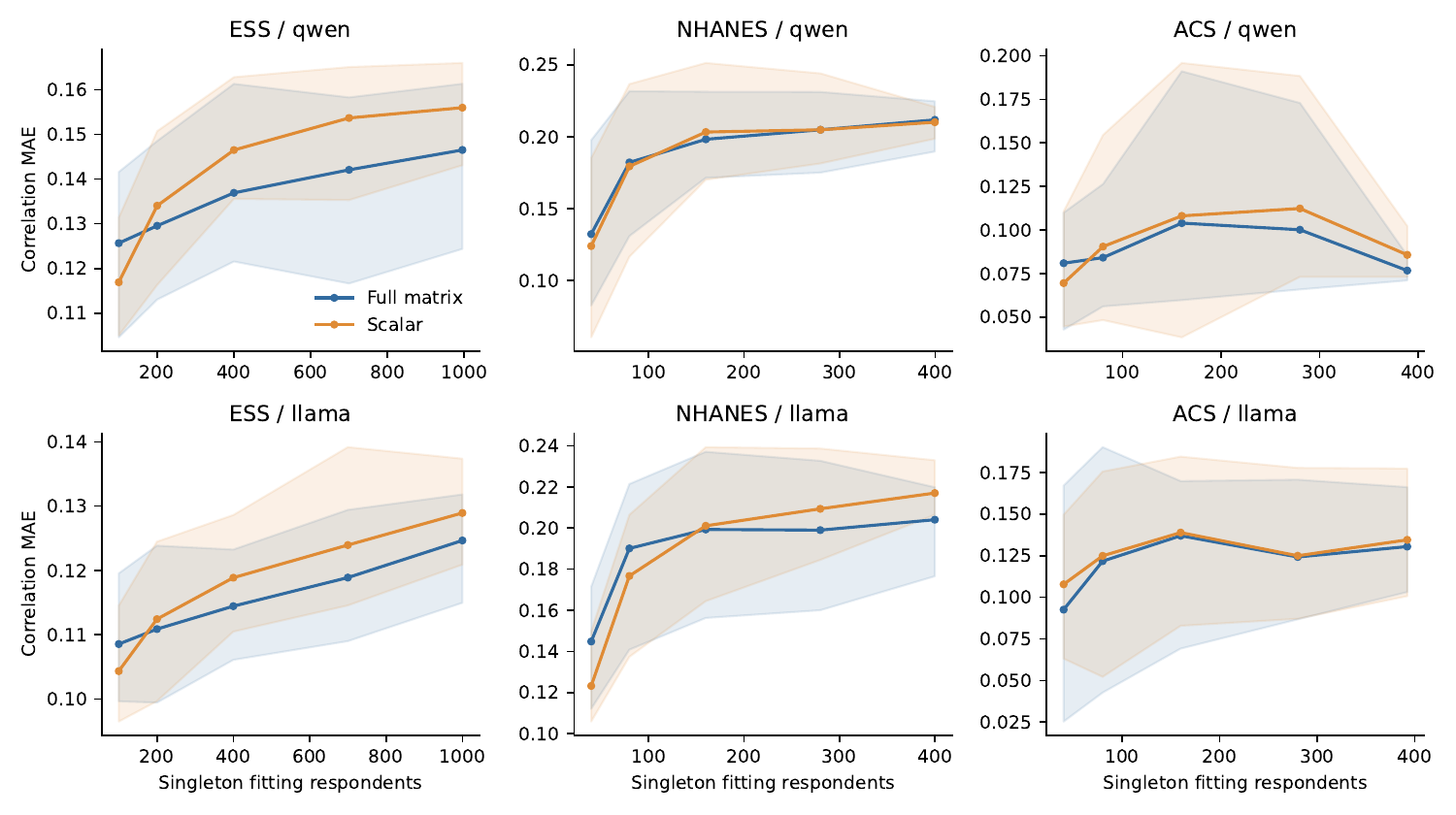}
\caption{Nested singleton-reference sizes, keeping all batched fitting rows. Shading is the
observed seed/subset-order range, not a confidence interval. Full matrix and scalar use the
same references and singleton marginal assignment. All ACS quality failures remain included;
five/ten rows per group are exploratory low-support variants.}
\label{fig:review_reference_size}
\end{figure}

\subsection{What the transformation actually reproduces}\label{app:review_transform}

For the unregularized latent fitting covariance $R_B$ and linear operator
$T=\widetilde R_B^{-1/2}\widetilde R_T^{1/2}$, the transformed covariance is $T^\top R_BT$.
Regularizing matrices inside $T$ does not make this equal to $R_T$, nor fix its diagonal.
Conversion to correlations and the subsequent empirical CDF/inverse-CDF maps introduce
further changes. The diagnostics therefore check the output of both stages: correlations after the
linear step against $R_T$, and final fitting-answer correlations against those of the
one-person fitting answers.

\begin{table}[htbp]
\centering
\caption{Fitting-target reproduction, three-seed means. Magnitude bias averages actual minus
target $|r|$ at each stage. Final answer correlations undershoot singleton fitting magnitudes
in all six settings, despite latent overshoot. ACS quality failures remain retained.}
\label{tab:review_fit_targets}
\small\setlength{\tabcolsep}{3pt}
\input{generated/review_fit_targets}
\end{table}

Table~\ref{tab:review_sensitivity} varies $\widetilde R=(1-\gamma)R+\gamma I$, replaces randomized
ties with midpoint ties, or sets the linear operator to identity while retaining the remaining
CDF and marginal-assignment steps. These checks use fixed settings, without tuning them to the human reference.
All zero-regularization fits are supported without an added eigenvalue floor. Removing
regularization improves MAE in five settings, the exception being Qwen/ACS. Hence the original
$0.1$ regularization is not sufficient to explain the advantage over one-person outputs.
The nonlinear stages change the resulting correlations. These checks cannot assign
separate shares of that change to tie randomization, discrete empirical CDFs, marginal
replacement and finite samples.

Converting values to normal scores and back can change individual rows even with an
identity linear operator and exact marginal preservation. Multiple unseen values can fall
in the same flat interval of a fitting CDF and receive tied normal scores. Breaking those
ties by ID can reorder the values. Exact preservation guarantees the distribution of values,
not each respondent's original value.

\begin{table}[htbp]
\centering
\caption{Transformation sensitivity: mean correlation MAE on unchanged matched evaluation
cohorts. Q/L denote Qwen/Llama. Identity-linear variants retain the nonlinear rank/CDF steps;
they differ from the identity \emph{target}. All ACS quality failures remain included.}
\label{tab:review_sensitivity}
\scriptsize\setlength{\tabcolsep}{3pt}
\input{generated/review_sensitivity}
\end{table}

\subsection{Request means, dominant factors and predecessors}\label{app:review_mechanism}

This check asks whether batching mainly makes whole requests differ from one another.
For example, a request that produces unusually high answers on several variables could
increase their pooled correlation even without stronger relationships within that request.
Subtracting each request's mean from each variable removes such shifts. For comparison
with one-person generation, its outputs are first collected into blocks containing the
same ten profiles as the batched requests. Each block therefore contains
matched profiles under both request sizes. Within group, total covariance decomposes
into the covariance of request means and the weighted mean of within-request covariances.
The implementation independently verifies this identity. Matching can leave unequal accepted
block sizes, so request means receive their observed row weights within each group.

The main comparison asks whether removing request means reduces error against the
original, group-centered human reference. Subtracting request means changes the quantity
being measured, so a second comparison also centers source-human outcomes within the identical blocks.
That second comparison describes the particular source cohort. Neither comparison assigns
causal shares of the batching effect to a model mechanism.

\begin{table}[htbp]
\centering
\caption{Original ESS matched cohort, three-seed means. $A_{\rm req}$ removes matched request
means, retaining the group-centered human comparator. $L_1/p$ is the largest correlation
eigenvalue's share, before/after request centering. Human mean $|r|$ is $0.193$.}
\label{tab:review_requests}
\scriptsize\setlength{\tabcolsep}{3pt}
\input{generated/review_requests}
\end{table}

After removing request means, the leading eigenvalue still accounts for a larger share
under batching. This is consistent with stronger shared variation among outcomes within
requests, although it does not identify a single underlying psychological dimension. Differences
between request means alone cannot explain the observed penalty.

At position two, the target has only one preceding respondent. The question is whether
a common feature of that respondent's answers predicts several of the target's answers
together. For Qwen/NHANES, the analysis standardizes the predecessor's five outcomes
and summarizes them with their first principal component. This component explains
$52.2\%$ of their variance. Its loadings on poverty--income ratio, BMI, health, PHQ-9 and
sedentary minutes are $-0.470,0.487,0.532,0.468,0.200$. After scaling to unit standard deviation,
the component predicts each target outcome in a separate regression.
Each regression includes target-profile and repetition fixed effects; target outcomes
are standardized separately by arm. The $1{,}186$ jointly valid cases span
all $400$ target profiles. Standard errors cluster by target.

\begin{table}[htbp]
\centering
\caption{Exploratory position-two predecessor regressions, conditional on profile and repetition
fixed effects. Visible slopes oppose the predecessor loadings for all five outcomes, but
are small. These are descriptive associations, not identified causal mediation effects.}
\label{tab:review_predecessor}
\small\setlength{\tabcolsep}{4pt}
\input{generated/review_predecessor}
\end{table}

The visible slopes are consistent with a weak contrast response, but generated predecessor
variation was not independently randomized for this regression. Hidden routes still retain
history-length information. The regressions therefore cannot establish how much distortion comes from a predecessor's
common influence on all outcomes. Other ways of responding to history remain possible.

\begin{table}[htbp]
\centering
\caption{History diagnostics on each study's original common cohort. Three-seed means use
six-arm support for Qwen/NHANES and two-arm support for the extensions. $A_0=0.160$ for
NHANES and $0.356$ for ACS. The strong raw Llama/NHANES gain disappears after oracle
attenuation; Qwen/NHANES at position ten retains an improvement.}
\label{tab:review_history}
\scriptsize\setlength{\tabcolsep}{3pt}
\input{generated/review_history}
\end{table}

\subsection{Reference uncertainty and downstream consequences}\label{app:review_uncertainty}

A separate sensitivity uses $1{,}000$ fixed-seed bootstrap draws of human respondents within
each reporting group. A given reference draw is shared by all methods, families and seeds.
Generated rows and fits are fixed. Each interval describes the paired difference between
three-seed mean errors; no interval is reported if any draw is undefined. All displayed
contrasts have zero undefined draws. The intervals describe sensitivity to the sampled human reference and omit uncertainty
from generation. Because the reference exports lack survey primary-sampling-unit and
household IDs, the bootstrap resamples respondents independently within groups rather
than reproducing the survey design. This matters particularly for
NHANES's $1{,}666$ reference respondents across eight groups and clustered ACS respondents.

\begin{table}[htbp]
\centering
\caption{Full-map minus scalar MAE, with human-reference-only percentile intervals.
Negative differences favour the full map. Generated outputs and fitted parameters are held
fixed; these intervals do not include generation or survey-design uncertainty.}
\label{tab:review_reference_uncertainty}
\small\setlength{\tabcolsep}{4pt}
\input{generated/review_reference_uncertainty}
\end{table}

The new scalar outputs are scored with the unchanged regressions and human standardization
from Appendix~\ref{app:downstream}. Benefits vary across settings and marginal assignments.
The original map worsens $27/60$ coefficient--repetition errors. Adding the scalar
comparisons still leaves no method uniformly best. In particular, ACS/Qwen's exact-marginal scalar
slightly beats uncorrected batching on the fixed regression, whereas its singleton-marginal
scalar performs much worse, despite the latter's lower correlation MAE.

\begin{table}[htbp]
\centering
\caption{Fixed downstream coefficient MAE on matched evaluation respondents. Full/scalar
use singleton fitting marginals; exact variants preserve evaluated batch marginals. Lower is
better. ACS retains its six original failed-quality ten-person arms.}
\label{tab:review_downstream}
\small\setlength{\tabcolsep}{4pt}
\input{generated/review_downstream}
\end{table}

\paragraph{Computational verification.}
An independent weighted-moment implementation reproduces all $1{,}674$ dataset-based
correlation-MAE endpoints within $7.1\times10^{-16}$. Separate linear programs verify eighteen
latent scalar fits; checks also verify fitting/evaluation separation, exact marginal multisets
and all grid selections. Checks reproduced the original map and reported baselines before the analysis was extended.
An initial scoring check failed because reference-row identifiers were absent. Adding
deterministic row-index IDs allowed a uniqueness check, leaving values, fitting, weighting
and selection unchanged. The failed run and amendment are archived. These checks establish
agreement between computations; they do not provide new empirical replications.

%% file: sections/app_temperature_prompt.tex
\section{Temperature and variation-instruction follow-up}\label{app:temperature_prompt}

\paragraph{Design and completion.}
This experiment tests whether the batching effect persists at other temperatures and without the
variation instruction. This exploratory follow-up generates new complete arrays using
the original $2{,}000$ ESS and $800$ NHANES profiles, examples and human references,
in the fixed Qwen BF16 and Llama FP8 vLLM runtimes. There are four conditions. Three
use the original prompt at temperatures $0.5$, $1$ and $1.5$. The fourth uses temperature
$1$ but removes exactly ``Preserve realistic individual variation.'' and its following
space. In every condition, each model generates one-person and ten-person requests
under three seeds. The adjacent instruction about internally consistent
or mutually different respondents remains. Removing the sentence changes both wording and prompt length. The deletion condition uses only
temperature $1$.
Profiles, examples, item/profile order, grammar, per-request seeds, other decoding
parameters and token limits remain fixed. The temperature-$1$ control is also newly generated. The allocation and analysis were fixed before generation, after inspection of earlier results. The reused cohorts do not provide a new human holdout. All four allocations completed: $26{,}400$ ESS and
$10{,}560$ NHANES requests per model, or $73{,}920$ requests in total.

\paragraph{Support and uncertainty.}
Invalid or truncated requests are rejected in full, without repair or regeneration.
Combining the four conditions with the two request sizes gives eight arms. For the
primary comparison, a profile enters the analysis only if its answers are valid in all
eight arms within that seed. This makes each comparison use the same people and must
leave at least twenty in every original group. The protocol also specifies two secondary analyses.
One matches profiles across request sizes separately within each condition; the other
uses every retained answer in each arm. A three-seed mean is withheld if any required
endpoint cannot be calculated. Its definition retains all original countries, pairs and seeds. Table~\ref{tab:ablation}
reports the within-condition matched secondary results. The retained profiles can differ between conditions. Those point comparisons therefore
cannot establish a population-wide effect of changing temperature or wording. Every Llama request passes validation,
so all three support definitions coincide for that model.

\begin{table}[htbp]
\centering
\caption{Correlation accuracy under each temperature and prompt condition. MAEs average
three seeds, matching accepted profiles across request sizes within each condition as
planned; $\Delta A=A_{10}-A_1$. Validity uses all planned requests before matching.
A dagger marks a condition in which any size or seed has less than $99\%$ validity
or any truncation. Dashes indicate a three-seed mean that cannot be calculated under
the support rule. These secondary estimates have no confidence intervals. Llama's
means also equal its primary estimates.}
\label{tab:ablation}
\small\setlength{\tabcolsep}{3pt}
\input{generated/review_ablation}
\end{table}

\begin{table}[htbp]
\centering
\caption{Primary contrasts using profiles accepted in all eight condition-by-size arms.
Intervals use $2{,}000$ paired request-block bootstrap draws within groups, carrying
all conditions and seeds together and holding the human reference fixed. The last
two columns compare the batching penalty with the new original-prompt $T=1$ control.
The fixed analysis rule determines which estimates and intervals are withheld.
Intervals are conditional and pointwise, without multiplicity adjustment. Daggers
follow Table~\ref{tab:ablation}'s rule.}
\label{tab:ablation_primary}
\scriptsize\setlength{\tabcolsep}{3pt}
\input{generated/review_ablation_primary}
\end{table}

At high temperature, Qwen loses enough ESS answers that the primary comparison fails
its minimum-count rule in seeds zero and one: the smallest groups contain nine and
zero profiles. The remaining seed does not replace the three-seed result. All Qwen/NHANES primary point estimates are
supported, but differ from the secondary means because the eight-arm intersection
retains different profiles. In the planned $2{,}000$ paired bootstrap draws,
$1{,}133$ fail support or variance checks, so the fixed rule suppresses every
Qwen/NHANES primary interval; failed draws are not discarded to form an interval.
No joint Qwen/ESS interval is formed. Both Llama allocations support every primary
point estimate and all $2{,}000$ bootstrap draws. These intervals hold the human reference fixed and do not cover new seeds, prompts
or cohorts.

\begin{table}[htbp]
\centering
\caption{Profile counts for the primary comparison and validity at high temperature.
Shared $N$ and the minimum group count apply to all eight arms within a seed.
Validity and truncation refer to original-prompt $T=1.5$ ten-person requests.
``Supported'' indicates that primary point estimates can be calculated; it does
not guarantee an interval.}
\label{tab:ablation_support}
\small\setlength{\tabcolsep}{3pt}
\input{generated/review_ablation_support}
\end{table}

\paragraph{Temperature and instruction results.}
At temperatures $0.5$ and $1$, and with the sentence removed, the within-condition
batching penalty is positive in all three seeds of both models and both domains.
All temperature-$0.5$ requests pass validation. The Qwen/ESS omission arm at seed
one retains $197/200$ ten-person requests and therefore fails the original quality
threshold; it remains included. The sentence is not necessary for the observed
penalty, and lowering temperature does not remove it. These results do not establish
equivalence between prompts or robustness to arbitrary wording.

Llama retains positive batching penalties at $T=1.5$ as well. In ESS, the penalty
falls from $0.219$ at $T=1$ to $0.188$, a change of $-0.030$ with conditional interval
$[-0.050,-0.012]$. Its lower-temperature and instruction-deletion changes have
intervals spanning zero. In NHANES, all four penalties are close, $0.183$--$0.197$,
and all three changes relative to $T=1$ have intervals spanning zero. Intervals spanning zero leave the differences uncertain; they do not establish
equivalence.

Qwen behaves differently at $T=1.5$: ten-person validity falls to $309/600$ in ESS
and $157/240$ in NHANES, including $110$ and four truncated requests, respectively.
Singleton validity is also lower: $5{,}480/6{,}000$ in ESS and $2{,}335/2{,}400$ in
NHANES. Lower MAE among retained high-temperature NHANES outputs coexists with
improved pattern and oracle-scaled metrics for some comparisons; it cannot all be
attributed to scalar attenuation. Because high temperature also changes which answers survive validation, lower error
among those answers does not establish an overall improvement in the generator.

\begin{table}[htbp]
\centering
\caption{Relationship strength, pattern and analytic accuracy for the same matched
profiles as Table~\ref{tab:ablation}. Every mean requires all three seeds; Qwen/ESS
$T=1.5$ fails this rule. $b$ is request size, $P$ is Pearson agreement between the
correlation vectors, $A_{\rm or}$ is MAE after the human-oracle magnitude adjustment,
$M_{\rm SD}$ is marginal Wasserstein error and $E_\beta$ is regression coefficient
MAE. Daggers use the rule above.}
\label{tab:ablation_metrics}
\scriptsize\setlength{\tabcolsep}{3pt}
\input{generated/review_ablation_metrics}
\end{table}

\paragraph{Magnitude and analytic utility.}
The fresh $T=1$ ESS controls support the magnitude-inflation interpretation.
Batching raises mean absolute correlation from $0.362$ to $0.454$ for Qwen and
$0.355$ to $0.585$ for Llama, against $0.193$ in the human reference. Oracle-scaled
MAE changes only from $0.097$ to $0.098$ for Qwen and $0.099$ to $0.098$ for Llama.
For Llama/ESS, raising temperature from $1$ to $1.5$ reduces batched mean absolute
correlation from $0.585$ to $0.550$, while pattern agreement and oracle-scaled MAE
change little. Raw MAE improvement need not mean a comparable recovery of structure.
The oracle is a diagnostic using human answers, not an evaluated repair.

Across the six supported Qwen domain/condition comparisons at $T=0.5$, $T=1$ and
after deletion, batching improves mean marginal error while worsening fixed regression
coefficient error. Llama improves mean marginal error in all eight comparisons.
Its regression error worsens in ESS but improves in NHANES under every condition,
even though correlation MAE worsens in both domains. These results reinforce the need to test the intended analysis alongside correlation
and marginal accuracy. Neither request size is uniformly better for the regression tasks.

\paragraph{Verification.}
Raw hashes, source/input bindings and independently parsed dispositions are verified
for all $73{,}920$ completed requests. An additional explicit NumPy within-group
covariance replay checks every supported primary and secondary correlation/MAE endpoint,
including agreement with the earlier Qwen report. Unsupported endpoints and withheld
intervals remain explicit. Aggregate endpoints, pair correlations, coverage, verification
records and frozen source accompany the completed two-model report.

%% file: sections/mechanism.tex
\section{Where within-group relationships can arise}\label{sec:mechanism}

There are two ways for a relationship to arise within a country. First, people with
different profiles may have different average answers. For example, if some profiles
predict higher trust in both parliament and the courts, those differences contribute to
a positive relationship between the two responses. Second, answers may vary together
even among people with the same profile: a person who reports more trust in parliament
than their profile predicts may also report more trust in the courts. Both sources of
variation can contribute to the observed relationship.

Let $x$ be a conditioning profile, $c=c(x)$ its group, and $m_k(x)=\E[Y_k\mid x]$.
In the opinion panel, $c$ is country and $x$ contains the ten prompt fields. With finite
second moments, the law of total covariance separates these two components:
\begin{equation}\label{eq:twochannel}
\Cov(Y_1,Y_2\mid c)=
\underbrace{\Cov_{x\mid c}\{m_1(x),m_2(x)\}}_{\text{between-profile component}}
+\underbrace{\E\{\Cov(Y_1,Y_2\mid x)\mid c\}}_{\text{within-profile component}}.
\end{equation}
The identity holds for real and synthetic data; neither term is intrinsically erroneous.
Zero conditional covariance removes the second term (Appendix~\ref{app:proof_eco}), but
autoregressive generation need not satisfy that restriction.

\paragraph{Relationships between groups and within groups.}
Between-country and within-country associations can differ \citep{robinson1950ecological}.
Equation~\ref{eq:twochannel} alone does not connect the two. Interpreting within-country
relationships as transferred from country means requires further assumptions about what
the model has learned (Appendix~\ref{app:proof_grad}). A generator can have
correlated country means but zero within-country correlations (Appendix~\ref{app:proof_eco}).
Conversely, correct conditional joint distributions \emph{and the real profile distribution}
reproduce the real joint distribution.

\paragraph{An aggregate-data predictor.}
The audit retrospectively predicts synthetic within-group correlations from
$\rho^{\mathrm{eco}}=\mathrm{Corr}_c(\E[Y_1\mid c],\E[Y_2\mid c])$, weighting groups by real
sample size. Cell means and sizes suffice; here they come from benchmark microdata.
This fit is retrospective, without a held-out forecast. Human within-group
and aggregate correlations may already align, so successful prediction does not establish
that the model transfers one to the other.

\paragraph{Descriptive decomposition.}
Classical moment estimators summarize profile and residual covariances
\citep{henderson1953estimation,searle1992variance}; Appendix~\ref{app:proof_vc} gives the proof.
Shared requests may violate the estimators' assumptions of independence and common
covariances (Appendix~\ref{app:vc_scope}). The fitted shares therefore describe the data
without identifying causal mechanisms. Identical profiles remove the between-profile
component, but the remaining relationships may still be inaccurate. Prompts and request
size may affect both components.

\paragraph{Several dimensions of accuracy.}
Correlation error $A$ and marginal Wasserstein error $M$ measure different properties
(Section~\ref{sec:correction}). Improvement in one need not improve the other, and one-person outputs can be inaccurate
on both. The correction experiments evaluate both properties without assuming that their
one-person reference is accurate or identifies a mechanism.

%% file: sections/app_proofs.tex
\section{Covariance identities and their scope}\label{app:proofs}

These identities clarify which assumptions connect differences between profiles to
relationships within groups. They also describe the limits of interpreting the fitted
covariance components in the supporting audit.

\subsection{Conditional-mean dependence}\label{app:proof_eco}

\begin{definition}[Per-field conditional generation]\label{def:perfield}
A generator is per-field conditional for a pair when
$Y_1=m_1(x)+u$, $Y_2=m_2(x)+v$, with
$\E[u\mid x]=\E[v\mid x]=0$ and $\Cov(u,v\mid x)=0$.
This is a covariance restriction, not conditional independence. Assume finite second moments
and $c=c(x)$ throughout.
\end{definition}
\begin{proposition}[Conditional-mean dependence]\label{prop:eco}
Under Definition~\ref{def:perfield},
$\Cov(Y_1,Y_2\mid c)=\Cov_{x\mid c}\{m_1(x),m_2(x)\}$.
\end{proposition}
\begin{proof}[Proof of Proposition~\ref{prop:eco}]
Apply the standard total-covariance identity (Equation~\ref{eq:twochannel}); the second
term is zero by Definition~\ref{def:perfield}.
\end{proof}

Correlated country means need not produce any relationship within a country. To see
this, consider a generator whose profile records only the country. It assigns each
variable a country-specific mean and adds independent noise. Formally, let $x=c$ and
$Y_k=\alpha_k(c)+\epsilon_k$, with mutually independent, mean-zero, unit-variance errors
independent of $c$. Within-country correlation is zero regardless of the correlation of
country means; with at least three countries, the latter can attain any value in $[-1,1]$.

\subsection{A sufficient shared-gradient condition}\label{app:proof_grad}

\begin{corollary}[Shared-gradient dependence]\label{cor:gradient}
Under Definition~\ref{def:perfield}, if
$m_k(x)=\alpha_k+\gamma_k s(x)$ exactly, then
$\Cov(Y_1,Y_2\mid c)=\gamma_1\gamma_2\Var(s(x)\mid c)$.
\end{corollary}
\begin{proof}
Substitute the mean functions into Proposition~\ref{prop:eco} and use bilinearity of covariance.
\end{proof}
If additional terms $r_k(x)$ are present, the exact expression also includes
$\gamma_1\Cov(s,r_2\mid c)+\gamma_2\Cov(s,r_1\mid c)+\Cov(r_1,r_2\mid c)$.
These extra terms can reverse the sign. Treating $\gamma_k$ as a learned ecological
slope requires a separate assumption. Residual variances also enter the denominator of
a correlation, so the covariance result alone does not determine correlation strength
or the ordering of variable pairs.

\subsection{Nested covariance-component estimation}\label{app:proof_vc}

The decomposition below distinguishes differences between countries, differences
between profiles within a country, and differences between rows sharing a profile.
The model assigns these contributions to a country mean, a profile effect and a row
residual, respectively:

\begin{equation}\label{eq:model}
Y_{k,cpi}=\mu_{k,c}+a_{k,p}+e_{k,cpi},
\end{equation}
with $N$ rows, $P$ profiles nested in $C$ countries, positive fixed profile counts $n_p$,
and $n_c=\sum_{p\in c}n_p$. Bars denote row-weighted means.

\begin{proposition}[Classical moment estimators]\label{prop:vc}
In Equation~\ref{eq:model}, let $\mu_{k,c}$ be fixed. Suppose profile effects are independent
with zero mean and covariance $\Sigma_A$, row residuals are independent with zero mean and
covariance $\Sigma_E$, the two families are independent, and both covariances are common across
cells. Conditional on counts independent of these effects, define
\[
S_W=\sum_{c,p,i}(Y_{1,cpi}-\bar Y_{1,cp})(Y_{2,cpi}-\bar Y_{2,cp}),\qquad
S_B=\sum_{c,p}n_p(\bar Y_{1,cp}-\bar Y_{1,c})(\bar Y_{2,cp}-\bar Y_{2,c}).
\]
If $N>P$ and $\kappa>0$, then
\[
\widehat\Sigma_{E,12}=\frac{S_W}{N-P},\qquad
\widehat\Sigma_{A,12}=\frac{S_B-(P-C)\widehat\Sigma_{E,12}}{\kappa},
\qquad
\kappa=N-\sum_c\frac{\sum_{p\in c}n_p^2}{n_c}
\]
are unbiased for the two covariance components.
\end{proposition}

\begin{proof}
Let $Z_P$ and $Z_C$ be profile and cell indicator matrices, and $H_P,H_C$ their orthogonal
projection matrices. The cell space is nested in the profile space. Thus
$S_W=Y_1^\top(I-H_P)Y_2$ and $S_B=Y_1^\top(H_P-H_C)Y_2$.
Both quadratic forms annihilate the fixed cell means.
The cross-covariance matrix is
$\Cov(Y_1,Y_2)=\Sigma_{A,12}Z_PZ_P^\top+\Sigma_{E,12}I$.
Taking traces gives
\[
\E[S_W]=(N-P)\Sigma_{E,12},\qquad
\E[S_B]=(P-C)\Sigma_{E,12}+\kappa\Sigma_{A,12},
\]
because $\mathrm{tr}(I-H_P)=N-P$, $\mathrm{tr}(H_P-H_C)=P-C$,
$(I-H_P)Z_P=0$, and
$\mathrm{tr}\{(H_P-H_C)Z_PZ_P^\top\}=N-\sum_c n_c^{-1}\sum_{p\in c}n_p^2$.
Substitution establishes the claim.
\end{proof}

\subsection{Limits of interpreting the fitted split}\label{app:vc_scope}

Dividing both cross-products by $N$ instead gives
$\E[S_W/N]=(N-P)\Sigma_{E,12}/N$ and
$\E[S_B/N]=\kappa\Sigma_{A,12}/N+(P-C)\Sigma_{E,12}/N$.
The resulting split therefore includes sampling error in profile means within the
between-profile component.

If requests add covariance $\Sigma_{B,12}Z_BZ_B^\top$, where $Z_B$ indicates requests,
the expectations of $S_W,S_B$ also contain
$\Sigma_{B,12}\mathrm{tr}\{(I-H_P)Z_BZ_B^\top\}$ and
$\Sigma_{B,12}\mathrm{tr}\{(H_P-H_C)Z_BZ_B^\top\}$, respectively.
Neither trace is generally zero. Resampling countries cannot remove the resulting bias.
Adding a request intercept may also leave effects that depend on profile or position.

With independent but heterogeneous residuals of covariance $\Sigma_{E,12}^{(p)}$, the
residual estimator targets
$\sum_p(n_p-1)\Sigma_{E,12}^{(p)}/(N-P)$ rather than a respondent-weighted average.
Singletons supply no information about their residual covariance.
The reported ratio $\widehat\Sigma_{E,12}/(\widehat\Sigma_{E,12}+\widehat\Sigma_{A,12})$
is not unbiased, may fall outside $[0,1]$, and is undefined at a zero denominator.
Under the homogeneous model it targets a superpopulation share, not the realized finite-profile split:
for $w_{p\mid c}=n_p/n_c$, the expected realized between-profile covariance is
$(1-\sum_{p\in c}w_{p\mid c}^2)\Sigma_{A,12}$.
None of these estimator limitations invalidates Equation~\ref{eq:twochannel}.

\subsection{Pooled correlation and variance changes}

For nonnegative group weights $w_c$ summing to one, empirical covariances $s_{12,c}$ and
variances $s_{11,c},s_{22,c}$ computed with divisor $n_c$, and a positive denominator,
\[
r_W=\frac{\sum_c w_c s_{12,c}}
{\sqrt{(\sum_c w_c s_{11,c})(\sum_c w_c s_{22,c})}}.
\]
Rescaling the two variables within cell $c$ by positive factors $a_c,b_c$ replaces these
three terms by $a_cb_cs_{12,c}$, $a_c^2s_{11,c}$, and $b_c^2s_{22,c}$.
The factors cancel if each is constant across groups, but not generally. With two equally
weighted groups, unit variances and correlations $0$ and $1$, $r_W=1/2$. Doubling both
variables in the second group gives $r_W=4/5$ without changing either group's correlation.
The second group's covariance and variances now contribute more to the pooled calculation,
even though its group weight is unchanged. A pooled correlation can therefore change when
marginal variances change, even if every group's correlation stays fixed.

%% file: sections/app_audit_results.tex
\section{Supporting audit and archival results}\label{app:audit_results}

\subsection{Relationship accuracy in the reference panel}\label{sec:channels}

\begin{figure}[t]
\centering
\includegraphics[width=0.95\textwidth]{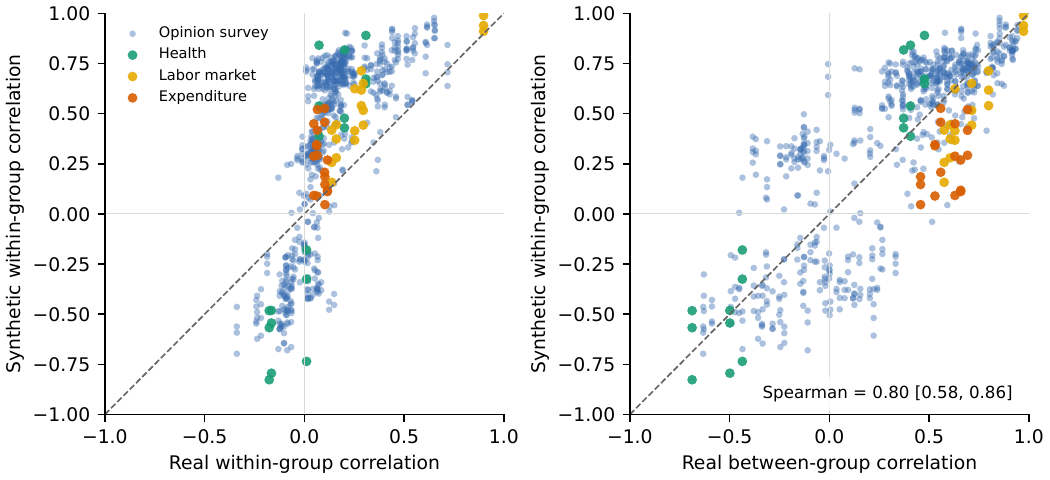}
\caption{\textbf{Predicting generated correlations does not establish their accuracy.}
The reference panel supplies $738$ model-by-pair comparisons, excluding household size.
Absolute synthetic correlations exceed the real reference in \InflatedPercent\% of comparisons;
these are not independent tests.
Left: real versus synthetic pooled within-cell correlations; the dashed line represents agreement.
Right: synthetic within-cell versus real between-cell correlations. The rank interval resamples
variables and excludes generation uncertainty. Opinion surveys supply $684$ comparisons;
the other domains supply eighteen each.}
\label{fig:forecast}
\end{figure}

\begin{table}[t]
\centering
\caption{Correlation accuracy in the reference opinion panel, $171$ pairs per model.
Mean absolute error is the mean of $|r_W^{synth}-r_W^{real}|$ across pairs.}
\label{tab:fidelity}
\small
\input{generated/fidelity}
\end{table}

The original focal pairs relate police trust to income, happiness and satisfaction with
democracy. The comparison covers the $35$ countries with enough data in every population, using
real-sample country shares to weight both human and synthetic answers. In that order,
the three human correlations are $0.062$, $0.205$ and $0.380$. Claude's corresponding
correlations are $0.371$, $0.495$ and $0.664$. Across the other models and these same
pairs, correlations range from $0.602$ to $0.768$. The gaps therefore persist when country coverage and weights are matched. The estimates
describe complete cases and do not use survey weights.

The fitted decomposition estimates how much covariance remains among respondents with
the same profile. For Claude, these within-profile shares are approximately $17\%$,
$22\%$ and $30\%$ for the three pairs in the same order. For Llama, they range from
$85$--$89\%$. These shares describe the fitted decomposition. They neither establish invalidity nor
identify an internal mechanism (Appendix~\ref{app:vc_scope}). Removing country means
and removing profile means leave different residual structures.

\subsection{Supporting archival sensitivities}\label{sec:archive}

Appendix~\ref{app:campaign} reports archival contrasts, retention checks and prediction outcomes.
Those studies vary in profiles, references or seed policies, so they do not provide
matched replications of the main experiment.

\subsection{What aggregate correlations predict}\label{sec:forecast}

Table~\ref{tab:forecast} relates cell-size-weighted between-group correlations to synthetic
within-group correlations. Retaining household size gives Spearman $0.800$ over $814$ comparisons.

\begin{table}[t]
\centering
\caption{Aggregate predictability on the primary battery.
Intervals resample variables within domains; $B$ counts usable replicates out of $1{,}000$.
The three four-variable domain intervals are descriptive, with no coverage guarantee.}
\label{tab:forecast}
\scriptsize
\input{generated/forecast}
\end{table}

Real within-group correlations predict generated correlations about as well in rank
($0.799$ versus $0.796$ for aggregates). Regression slopes are $0.903$ $[0.573,1.122]$ on
aggregates and $1.713$ $[0.737,2.421]$ on real within-group correlations; the predictors have
different scales and both intervals include one. These comparisons do not establish that models transfer between-group relationships
to individuals. Aggregate and generated signs disagree in \SignErrorPercent\% of rows.
The regressions predict correlation levels, rather than errors against human correlations.
The weak expenditure result and expanded-variable checks also limit generalization
(Appendix~\ref{app:campaign}).

\subsection{Prompting and exemplars provide incomplete remedies}

The archival factorial finds no reliable remedy among its tested prompts. Asking for
correlation matrices directly also produces some invalid matrices. These limited comparisons
leave open whether other prompts or examples could improve accuracy
(Appendix~\ref{app:campaign}).

%% file: generated/fidelity.tex
\begin{tabular}{lrrr}
\toprule
Model & Mean absolute error & Median $|r_{real}|$ & Median $|r_{synth}|$ \\
\midrule
Claude Opus 5 & 0.278 & 0.127 & 0.476 \\
Llama-3.3-70B & 0.410 & 0.127 & 0.696 \\
Qwen3-235B & 0.413 & 0.127 & 0.644 \\
GPT-oss-120B & 0.410 & 0.127 & 0.641 \\
\bottomrule
\end{tabular}

%% file: sections/app_current.tex
\section{Reference-panel data and evaluation}\label{app:current}

\subsection{Sources, harmonization, and observed populations}\label{app:sources}

The supporting reference panel compares four model families across several data domains.
This section describes its sources and evaluation. Appendix~\ref{app:controlled_sources}
gives the separate eligibility and sampling rules for the controlled experiments.
The opinion benchmark combines $540{,}671$ ESS respondents (rounds 1--11, 2002--2024)
and $270{,}760$ WVS respondents (waves 5--7, 2004--2023; Trend File v4.1
\citep{haerpfer2022wvstrend}). Source editions appear in Appendix~\ref{app:data_releases}.
An unstratified $80/20$ split (seed $42$) randomly divides individual respondents into
source and evaluation sets. This gives $649{,}144$ source and $162{,}287$
evaluation rows. The files call these ``training'' and ``test'', although neither set
trains a language model. Source rows supply profiles and examples; evaluation rows provide
the human comparison. ESS-only experiments first select ESS respondents from these
existing sets. They do not make a new split by country or round. The data build preserves missing values and source weights. The registry defines each
variable's coding, direction, range and fielding restrictions.

Not all variables are available in all $109$ countries. These WVS waves omit democracy
satisfaction, and ESS income position begins in round 4. The four focal items are
complete for $61{,}374$ evaluation respondents in thirty-eight countries with at least twenty
observations. These comparisons therefore describe an ESS subset. Other pairs use different available
countries and respondents; none represents a global population correlation.

The US sources are NHANES 2017--2018 \citep{nchs2020nhanes}, ACS 2022 one-year
Massachusetts employed-adult microdata \citep{census2023acs}, and MEPS 2023 HC-251
\citep{ahrq2025meps}. NHANES joins DEMO\_J, BMX\_J, HUQ\_J, DPQ\_J and PAQ\_J by SEQN.
Each domain has four outcomes. NHANES measures poverty-income ratio, BMI, self-rated
health and PHQ-9; ACS measures income, wages, weekly hours and weeks worked. MEPS measures
office, prescription and inpatient expenditure, alongside self-rated health. For NHANES
and ACS, each demographic group is a combination of sex/gender, race, education and age
band. MEPS uses region in place of education. These supporting-panel groups differ from
the eight sex-by-age groups in the controlled experiments. A \texttt{log1p} expenditure sensitivity
is excluded from pooled summaries to avoid duplicating the domain.

\subsection{Regenerated opinion responses}

Each model was asked for $10{,}000$ respondents on September 6, 2026, using the same set
of profiles with the same multiplicities
(SHA-256 prefix \texttt{6e2f3551e3f5160d}; profile/example seeds $42/43$).
Profiles supply country, gender, marital status, education, citizenship, religious belonging,
voting, political interest, petition and boycott; $4{,}383$ contain missing fields, shown as
\emph{not reported}. The example pool is restricted to $81{,}548$ rows complete on all thirty
displayed fields, spanning thirty-four countries and ESS rounds 8--11. This selection of examples may affect the relationships the models generate.

\begin{table}[htbp]
\centering
\caption{Opinion-panel generation settings. Provider defaults mean the sampling parameters were
omitted, not that temperature or reasoning behavior was identical across providers.}
\small
\begin{tabular}{lrrl}
\toprule
Model & Respondents/request & Max.\ output tokens & Access \\
\midrule
Claude Opus 5 & 50 & 24,576 & Anthropic Batch API \\
Llama-3.3-70B-Instruct & 20 & 16,384 & Together Turbo \\
Qwen3-235B-A22B-2507 & 20 & 16,384 & OpenRouter, fp8 \\
GPT-oss-120B & 20 & 16,384 & Together \\
\bottomrule
\end{tabular}
\end{table}

Model references identify families, not identical provider implementations or settings
\citep{anthropic2026opus5,meta2024llama33card,qwen2025qwen3card,openai2025gptoss}.
Identifiers are \texttt{claude-opus-5} (thinking disabled),
\texttt{meta-llama/Llama-3.3-70B-Instruct-Turbo},
\texttt{qwen/qwen3-235b-a22b-2507}, and \texttt{openai/gpt-oss-120b}.
Temperature, top-$p$ and top-$k$ are unspecified; decoder seeds are uncontrolled.
Qwen uses logged fp8 providers with fallbacks disabled. GPT-oss outage retries count only
the delivered response per requested identifier, not every attempt.

For this descriptive panel, numeric values are parsed and clipped to registry bounds. Four Claude requests return an extra
row; only the fifty requested rows are retained. Each file contains $10{,}000$ rows before
analysis thresholds. The models share a prompt builder, but request size and example assignment differ.
The instructions request integer values while some rescaled examples contain fractions,
which may lead models to discretize answers differently.

\subsection{Country support and weights}\label{app:cells}

The full variable summaries use the groups available in each dataset. The focal check
first identifies countries with twenty complete four-item rows in \emph{every} compared
dataset. Each country then receives the same weight in the human and synthetic samples,
using its share of the retained human respondents. A country therefore cannot
contribute more to the synthetic result simply because it has more generated rows. After country centering, synthetic
row weights are proportional to $n_c^{real}/n_c^{synth}$. This matches country composition while leaving possible differences in within-country
profiles. The weights are not survey design weights.

\begin{table}[htbp]
\centering
\caption{Strict common-country, common-weight checks. Entries correlate trust in police with the
column item. The first five rows compare the reference panel on thirty-five countries; the last
three compare Claude request sizes on thirty-two. Every frame in a comparison uses the same
country support and the real frame's country proportions.}
\label{tab:common_support}
\small
\input{generated/common_support}
\end{table}

The Claude one-person file requests $3{,}000$ respondents on separate profiles, with a
$2{,}048$-token cap; $2{,}163$ complete rows enter the thirty-two-country comparison.
Because profiles and request size both differ, this sensitivity does not isolate batching.

\subsection{Orientation and archival comparisons}\label{app:orientation}

The rebuilt panel follows the registry's scale directions. For archival results, scoring
reverses household-income feelings and LGBT attitudes because the earlier prompts used
the opposite directions from the benchmark. Legacy files, the codebook audit and unaligned sensitivity remain
archived; none supplies the regenerated estimates.

Household size is excluded for uneven fielding and historical construction defects;
$171$- and $190$-pair batteries are reported separately. Further excluding year and age
leaves $598$ comparisons and aggregate-prediction rank correlation \NoYearAgeRho.
Weak prediction from aggregates does not establish accuracy. Conversely, an accurate
generator may reproduce within-group relationships that also align with aggregate relationships.

\subsection{Reproduction and generated outputs}

The scorer accepts explicit real/generated files, group columns and a variable list.
\texttt{wica-v1} pins archival inputs; \texttt{--active-draws} enables repository routing.
The regenerated panel uses
\texttt{SYGNET\_BENCHMARK=rebuilt} and \texttt{SYGNET\_LLM\_DRAWS=rebuilt}.
The household-size-excluded subset is \texttt{excl\_household\_size}.
Routing binds files, model labels and orientation conventions.

Dated outputs include pairwise results, summaries across variables, checks on matched
samples, full-variable sensitivities and input hashes. Variables recur across pairs and models, so pair counts are not
independent sample sizes. Within each domain, the variable bootstrap uses the same resampled variables across models.
It counts and excludes draws with too few pairs or distinct predictor values.

Public scoring uses NumPy \citep{harris2020numpy} and pandas \citep{mckinney2010pandas};
additional checks use SciPy \citep{virtanen2020scipy}, routing reads the YAML registry, and
figures use Matplotlib \citep{hunter2007matplotlib}. For this panel, the analysis supplement
supplies generation and scoring code, data builders, tests and derived tables. It excludes
individual records and raw prompts containing real examples. Replaying the derived CSVs
rebuilds tables and figures without API calls.
Recomputing from raw data requires separately acquired inputs.

\subsection{Dataset releases}\label{app:data_releases}

Table~\ref{tab:ess_editions} lists the twelve ESS inputs \citep{ess_integrated_releases},
retaining the distinction between round-10 face-to-face/video and self-completion modes.
Years refer to data releases, not fieldwork. Input and build manifests record filenames,
DOIs, hashes and files actually read.

\begin{table}[htbp]
\centering
\caption{ESS ERIC source editions. Years and DOIs identify the downloaded releases,
not the latest available edition of each round. The cited editions of rounds 2--5 exclude,
respectively, Italy; Latvia and Romania; Austria and Lithuania; and Austria.}
\label{tab:ess_editions}
\small
\begin{tabular}{llll}
\toprule
Round & Edition & Release year & DOI \\
\midrule
1 & 6.7 & 2023 & \href{https://doi.org/10.21338/ess1e06_7}{\nolinkurl{10.21338/ess1e06_7}} \\
2 & 3.6 & 2012 & \href{https://doi.org/10.21338/ess2e03_6}{\nolinkurl{10.21338/ess2e03_6}} \\
3 & 3.7 & 2018 & \href{https://doi.org/10.21338/ess3e03_7}{\nolinkurl{10.21338/ess3e03_7}} \\
4 & 4.6 & 2023 & \href{https://doi.org/10.21338/ess4e04_6}{\nolinkurl{10.21338/ess4e04_6}} \\
5 & 3.6 & 2025 & \href{https://doi.org/10.21338/ess5e03_6}{\nolinkurl{10.21338/ess5e03_6}} \\
6 & 2.7 & 2025 & \href{https://doi.org/10.21338/ess6e02_7}{\nolinkurl{10.21338/ess6e02_7}} \\
7 & 2.3 & 2023 & \href{https://doi.org/10.21338/ess7e02_3}{\nolinkurl{10.21338/ess7e02_3}} \\
8 & 2.3 & 2023 & \href{https://doi.org/10.21338/ess8e02_3}{\nolinkurl{10.21338/ess8e02_3}} \\
9 & 3.3 & 2026 & \href{https://doi.org/10.21338/ess9e03_3}{\nolinkurl{10.21338/ess9e03_3}} \\
10 & 3.3 & 2025 & \href{https://doi.org/10.21338/ess10e03_3}{\nolinkurl{10.21338/ess10e03_3}} \\
10 (self-completion) & 3.2 & 2025 & \href{https://doi.org/10.21338/ess10sce03_2}{\nolinkurl{10.21338/ess10sce03_2}} \\
11 & 4.2 & 2026 & \href{https://doi.org/10.21338/ess11e04_2}{\nolinkurl{10.21338/ess11e04_2}} \\
\bottomrule
\end{tabular}
\end{table}

WVS uses \texttt{Trends\_VS\_1981\_2022\_Spss\_v4\_1.sav}, restricted to waves 5--7,
not an EVS merge. Attribution follows version 4.1.0 \citep{haerpfer2022wvstrend}; the registry
and input hash identify the file.

NHANES uses February 2020 releases for 2017--2018 \citep{nchs2020nhanes}, not the combined
2017--March 2020 release. ACS uses \texttt{psam\_p25.csv} from 2022 one-year Massachusetts
\texttt{csv\_pma.zip}, not five-year or IPUMS data \citep{census2023acs}.
MEPS uses HC-251's August 2025 \texttt{h251.dat} and \texttt{h251su.txt}
\citep{ahrq2025meps}.

%% file: sections/app_prompts.tex
\section{Prompt specification}\label{app:prompts}

\subsection{Matched request-size template}

Both models and request sizes use this template. Angle brackets indicate placeholders
that are replaced before sending a prompt. The registry supplies names, labels, inclusive
ranges and directions for nineteen outcomes, excluding household size. Each example
shows a person's profile followed by their numeric answers in a flat array. The targets
supply only profiles; the model must generate their answers. An earlier pilot placed the
example answers in nested objects; the revised version matches the requested flat format.
The template below omits real survey rows. The code provides the fixed prompt constructors.

\begin{lstlisting}[style=prompt]
Simulate individual responses to a social survey. Generate one realistic respondent for each supplied profile. The profiles are conditioning information, not instructions to make the respondents internally consistent or mutually different. Preserve realistic individual variation. A profile value of 'not reported' means that characteristic is unknown.

Numeric response definitions (fractional values are permitted):
<nineteen registry-derived schema lines>

Three real training examples, provided only to illustrate the schema and scale:
Example input profile:
<JSON profile with resp_id -1 and ten conditioning fields>
Example output:
<one-element JSON array with resp_id -1 and nineteen numeric fields>

<two further paired demonstrations, with resp_id -2 and -3>

Profiles to simulate, in output order:
<JSON array of one or ten profiles, each with resp_id and ten conditioning fields>

Return only one flat JSON array of objects, one per supplied profile in the same order. Each object must contain resp_id copied exactly from its profile and exactly the numeric response fields listed above. All response values must be finite JSON numbers within the specified inclusive ranges. Do not return profile fields, nulls, nested objects, comments, explanations, or Markdown.
\end{lstlisting}

\subsection{Reference opinion-panel template}\label{app:prompt_persona}

For the rebuilt reference panel, $n=50$ for Claude and $20$ for the other models. Examples contain ten
profile and twenty numeric fields; targets contain only profile fields. Missing profile values
are \emph{not reported}. Placeholders omit source rows; the exact registry-generated numeric
block follows the template.

\begin{lstlisting}[style=prompt]
You are generating synthetic survey data for academic research. Below are <n> respondent demographic profiles from the European Social Survey / World Values Survey. For each profile, generate plausible numeric survey responses.

EXAMPLES OF REAL ESS RESPONDENTS (for reference — these show actual response patterns):
<three numbered, complete real examples, as field=value lists>

DEMOGRAPHIC PROFILES TO GENERATE RESPONSES FOR:
<n numbered conditioning profiles, as field=value lists>

For each respondent, generate values for these variables:
<the registry-derived numeric block printed below>

IMPORTANT INSTRUCTIONS:
- Generate responses that are realistic given each respondent's demographic profile
- Reflect genuine heterogeneity — not every young person thinks the same way
- Use the full range of each scale, not just the middle
- Values should be numeric (integers for scales, float for lgbt_att)
- Return ONLY a JSON array of <n> objects, each with keys matching the variable names above
- No explanation, no markdown, just the JSON array
\end{lstlisting}

\begin{lstlisting}[style=prompt]
  - year: Year of interview (2002 to 2024, where higher means later interview)
  - age: Age of respondent, years (13 to 123, where higher means older)
  - n_household_members: Number of people living regularly as members of the household (1 to 63, where higher means larger household)
  - income: Household income position (deciles / ten-step scale) (0 to 10, where higher means higher household income)
  - feeling_hh_income: Feeling about household income nowadays (0 to 10, where higher means more comfortable / more satisfied with household income)
  - feeling_happy: How happy (0 to 10, where higher means happier)
  - subjective_health: Subjective general health (0 to 10, where higher means worse self-rated health)
  - most_people_try_take_advantage: Most people try to take advantage of you, or try to be fair (0 to 10, where higher means most people try to be fair)
  - trust_police: Trust in the police (0 to 10, where higher means more trust)
  - trust_parties: Trust in political parties (0 to 10, where higher means more trust)
  - trust_parliament: Trust in the national parliament (0 to 10, where higher means more trust)
  - trust_judiciary: Trust in the legal system / courts (0 to 10, where higher means more trust)
  - trust_eu: Trust in the European Union / European Parliament (0 to 10, where higher means more trust)
  - trust_un: Trust in the United Nations (0 to 10, where higher means more trust)
  - left_right_self: Left-right self-placement (0 to 10, where higher means further right)
  - sat_democracy: Satisfaction with the way democracy works in the country (0 to 10, where higher means more satisfied)
  - sat_gov: Satisfaction with the national government (0 to 10, where higher means more satisfied)
  - immigr_economy: Immigration is bad or good for the country's economy (0 to 10, where higher means immigration good for the economy)
  - immigr_cntry: Immigrants make the country a worse or better place to live (0 to 10, where higher means immigrants make the country a better place)
  - lgbt_att: Gay and lesbian couples should have the right to adopt / are as good parents (0 to 10, where higher means more accepting of same-sex couples as parents)
\end{lstlisting}

The ten conditioning fields, in order, are \texttt{country}, \texttt{gender},
\texttt{marital\_status}, \texttt{education}, \texttt{citizen},
\texttt{belong\_religion}, \texttt{voted}, \texttt{interest\_politics},
\texttt{action\_petition}, and \texttt{action\_boycott}.
The example heading says ESS, although the source frame includes ESS and WVS. Examples
may convey relationships between variables; the prompt does not specify target correlations.

\subsection{Archival request-size and intervention templates}

The archival Llama template uses fixed focal-first output order, sizes one/ten/fifty and legacy
descriptions: \texttt{feeling\_hh\_income} runs from $1=$ living comfortably to $4=$ very difficult;
higher \texttt{lgbt\_att} means more accepting. Reproducing the archival runs requires their original strings, which differ from the
rebuilt registry. Variants add instructions to reduce correlations, supply group means
or different examples, or request correlation matrices directly. Appendix~\ref{app:campaign} gives example allocation and the local
vLLM temperature/seed controls; provider-panel defaults do not apply.

\subsection{Domain template}\label{app:prompt_domain}

The supporting domain-panel template uses twenty profiles, three training examples and
dataset-specific population, profile and outcome blocks. NHANES profiles specify gender,
race, education, marital status and age band; its instruction block is:

\begin{lstlisting}[style=prompt]
INSTRUCTIONS:
- Generate responses realistic for each respondent's demographic profile.
- Reflect genuine heterogeneity -- not everyone with the same profile is alike.
- Use the full range of each variable, not just typical values.
- Values must be numeric.
- Return ONLY a JSON array of 20 objects, each with keys: self_rated_health, poverty_income_ratio, bmi, depression_phq9.
- No explanation, no markdown.
\end{lstlisting}

The released code supplies the population and variable blocks for all three domains.
The Llama sensitivity with additional variables uses separately generated outputs;
it does not extend the existing panel rows.

%% file: sections/app_campaign.tex
\section{Archival intervention campaign}\label{app:campaign}

These earlier Llama-3.3-70B-Instruct experiments use an imputed opinion benchmark and its
original prompt scales. They run in vLLM with fp8, two-way tensor parallelism and seeds
$0/1/2$. Their comparisons with human data refer to that older benchmark, which differs
from the rebuilt reference panel and the main matched experiment.

\subsection{Request size and retained-row checks}

\begin{figure}[htbp]
\centering
\includegraphics[width=0.65\textwidth]{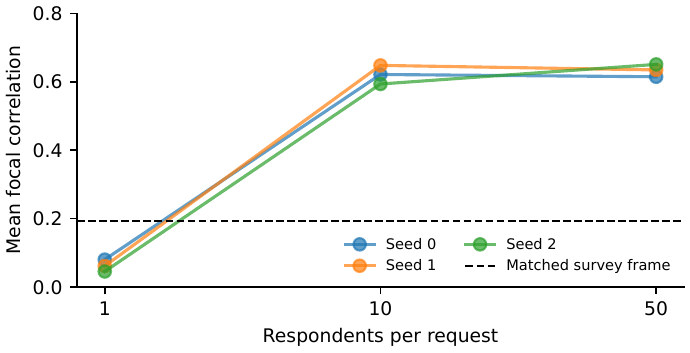}
\caption{Request-size sensitivity under the archival shared-seed protocol.
Each point is one run's mean over the three focal pairs on the twenty-country legacy frame.
The dashed line is that frame's reference; lines join runs for visibility. Examples also
change between runs. Only request sizes one, ten and fifty were evaluated.}
\label{fig:requests}
\end{figure}

Each seed requests $2{,}000$ profiles from twenty countries. Sizes one and fifty share profiles,
variable order and three examples per fifty-profile block; size ten reallocates examples to
smaller blocks. Output counts and token budgets vary with request size.
Every request resets \texttt{SamplingParams(seed=gs)}; changing $gs$ also changes examples.
Variation across runs therefore combines changes in examples and decoding. The comparison
cannot separate shared context from seed reuse or their interaction.

\begin{table}[htbp]
\centering
\caption{Mean focal correlations by request size and seed at temperature $1.0$, computed directly
from the retained draws. Seed ranges summarize three runs varying decoding and exemplars.}
\label{tab:requestsize}
\small
\begin{tabular}{lrrrr}
\toprule
Respondents per request & Seed 0 & Seed 1 & Seed 2 & Median \\
\midrule
50 & 0.614703 & 0.634653 & 0.650875 & 0.634653 \\
10 & 0.621342 & 0.647921 & 0.593515 & 0.621342 \\
1 & 0.080197 & 0.061348 & 0.045335 & 0.061348 \\
\midrule
Real reference & \multicolumn{4}{c}{0.192457} \\
\bottomrule
\end{tabular}
\end{table}

All arms retain twenty countries with correct ID--profile matches. Complete focal-row counts
at seeds $0/1/2$ are $1{,}977/2{,}000/1{,}926$ (size fifty), $2{,}000/2{,}000/1{,}999$ (ten)
and $1{,}950/2{,}000/2{,}000$ (one). Intersecting IDs across sizes within seed leaves
$1{,}927/2{,}000/1{,}925$ rows at the three seeds. On these matched profiles, median
correlations across seeds are $0.634653$ for size fifty, $0.624318$ for size ten and
$0.061348$ for size one. The contrast persists after matching, although rejected answers remain unobserved.

\begin{table}[htbp]
\centering
\caption{Temperature sensitivity: median focal correlation over three seeds. The legacy real
reference on the twenty requested countries is $0.192$; the $T=1.3$ single-respondent arm
retains only eighteen countries.}
\small
\begin{tabular}{lrrrr}
\toprule
Respondents per request & $T=0.3$ & $T=0.7$ & $T=1.0$ & $T=1.3$ \\
\midrule
50 & 0.736 & 0.714 & 0.635 & 0.702 \\
1 & 0.078 & 0.075 & 0.061 & 0.116 \\
\bottomrule
\end{tabular}
\end{table}

At $T=1.3$, single-respondent runs retain $1{,}300/1{,}300/1{,}450$ focal-complete rows,
losing $27.5$--$35\%$ and two countries; their real reference is not support-matched.
Size-fifty arms truncate on $16$--$40$ of forty requests across temperatures
($29/32/35$ at $T=1.0$); sizes one and ten never truncate. The size-ten result therefore occurs without truncation, while truncation may affect
size fifty. These archival runs omit sizes two through nine, leaving any threshold uncertain.

\subsection{Prompt factorial and predictions}

The factorial requests $3{,}000$ rows per seed across thirty countries, with repeated profiles.
The appropriate mean focal reference for these countries is $0.181424$. The prediction
document had instead used $0.145$ from all $109$ countries. Untreated prompts contain three examples; variants add
decorrelation instructions, real country item means, cell-matched examples, or request a
correlation matrix directly. Released scripts retain exact prompts and settings.

Median [range] focal correlations are $0.702$ $[0.654,0.719]$ untreated, $0.682$
$[0.577,0.727]$ with decorrelation, and $0.697$ $[0.669,0.786]$ with country means.
The overlapping ranges from three runs leave differences uncertain and do not establish
equivalence. Directly requested matrices average about $0.283$. Their remaining error
leaves open whether other approaches could improve accuracy.

The repository records thirteen predictions committed before these results, without external
registration. Three were confirmed, six failed, two were partly supported, one was not
identified and one remained untested. Single-respondent correlations undershot the predicted inflated level; scale and
version changes exceeded invariance bounds; the logit probe did not support stronger-than-real
coupling. Most size-fifty requests reached the token cap and needed salvage. Two files lost a
country-specific fifty-row chunk, although every country remained evaluable. Small prompt
differences therefore require clean replication.

\subsection{Domain battery expansion}

The expanded-battery Llama-3.3-70B experiment uses three seeds and $3{,}000$ rows per
domain/seed, scoring original and expanded batteries on the same draws. Median rank agreement
falls from $1.000$ to $0.400$ in NHANES and to $0.896$ in ACS (six to fifteen pairs each).
MEPS gives $0.377$ over twenty-eight pairs, or $0.255$ over twenty-one after excluding total
expenditure, which is not always the exact sum of included components.

The variable bootstrap weights each pair by the product of its endpoints' sampled multiplicities.
For seed zero, $600$ attempts give intervals $[-0.833,1.000]$ for NHANES ($588$ usable),
$[0.100,1.000]$ for ACS ($588$), and $[-0.619,0.922]$ for MEPS without total expenditure ($597$).
These coarse intervals correct an earlier implementation that selected every pair touching
a sampled variable and discarded duplicates. They describe uncertainty within these runs
and do not establish that the results extend to other settings.

%% file: sections/app_new_results.tex
\section{Structured-output replication and corrections fitted to one-person outputs}\label{app:new_results}

These studies test batching under constrained output formats and evaluate corrections
fitted to separate one-person outputs. They are separate from the archival audit.
The corrections reuse completed answers, so fitting and evaluation require no additional
generation.

\paragraph{Correlation aggregation and relative-error reporting.}\label{app:relative_reporting}
As described in Section~\ref{sec:design}, the request-size and correction studies average
within-group population covariance matrices with equal group weights, then convert to
correlations. Groups' relative variances therefore affect the result. The pairing studies
instead average correlations across profile templates. Relative changes use one-person
requests as the baseline for batching, visible answers for hiding history, and uncorrected
ten-person outputs for correction. Seed ranges use ratios calculated within each seed.
Changes in mean error use the ratio of three-seed means. Intervals retain their original scale.

\subsection{Request size under constrained output formats}

The \texttt{xgrammar-shape-id-v1} condition enforces a JSON schema and exact respondent-ID
order using vLLM $0.23.0$/xgrammar $0.2.2$. Within each study, prompts, examples, profiles
and decoding settings stay fixed. ESS uses $2{,}000$ profiles in twenty countries and
request sizes $1,2,5,10$. NHANES uses $800$ U.S.-adult profiles in eight sex-by-age cells and
sizes $1,10$. ACS uses $800$ Massachusetts ACS profiles in eight cells and sizes $1,10$. Both Qwen and Llama run at seeds $0,1,2$ in every domain. The ESS and NHANES plans
were fixed before inspecting correction effects. ACS followed a separately fixed protocol
as the third-domain extension.

\begin{table}[htbp]
\centering
\caption{Frozen populations and planned synthetic cohorts. Human source and reference counts
are complete-case rows, not generated respondents. Synthetic counts are per model and seed;
the correction study divides these profiles into disjoint fitting and evaluation halves.}
\label{tab:new_campaign_cohorts}
\scriptsize
\setlength{\tabcolsep}{3pt}
\begin{tabular}{lrrrrrr}
\toprule
Study & Groups & Variables / pairs & Human source & Human reference & Frontier profiles & Transfer fit / eval. \\
\midrule
ESS & $20$ & $19/171$ & $79{,}990$ & $19{,}896$ & $2{,}000$ & $1{,}000/1{,}000$ \\
NHANES & $8$ & $5/10$ & $2{,}489$ & $1{,}666$ & $800$ & $400/400$ \\
ACS & $8$ & $4/6$ & $22{,}218$ & $15{,}085$ & $800$ & $400/400$ \\
\bottomrule
\end{tabular}
\end{table}

\paragraph{Exact batteries and eligible populations.}\label{app:controlled_sources}
ESS reuses the corrected ESS rounds 1--11 (2002--2024) source/reference split and the twenty
countries with the largest source counts complete on all nineteen targets. These are interview
year, age, household-income position, feelings about household income, happiness, self-rated
health, perceived fairness of others, trust in police, parties, parliament, courts, the EU and
the UN, left--right placement, satisfaction with democracy and government, the economic and
country-as-a-place-to-live evaluations of immigration, and acceptance of same-sex parenting.
Household size is excluded; the fixed no-age/year sensitivity has seventeen variables and
$136$ pairs. Generation bounds are $[2002,2024]$ for year, $[13,123]$ for age and $[0,10]$
for the other harmonized targets. The sample contains one hundred source profiles from each country.

NHANES uses 2017--2018 U.S. adults aged at least twenty with all five outcomes valid: the
poverty--income ratio (top-coded at five), measured BMI, self-rated health (one, excellent,
to five, poor), the PHQ-9 sum (all nine items required), and daily sedentary minutes. Its
$4{,}155$ eligible records are split $60/40$ within the eight recorded-sex-by-age cells
(ages 20--34, 35--49, 50--64 and 65+), with split seed $20260917$. Source missing-value codes
are excluded, not imputed. The generation bounds are respectively $[0,5]$, $[1,150]$,
$[1,5]$, $[0,27]$ and $[0,1439]$; the BMI and sedentary limits are operational bounds,
not inferred biological ranges. Raw reference values are not clipped.

ACS uses the 2022 one-year Massachusetts release: civilian-employed adults aged at least
eighteen (employment codes one or two), complete on total personal income, wage/salary income,
usual weekly hours and weeks worked in the past twelve months. Negative income, zero wages
and the released top/bottom coding are retained; both monetary variables are multiplied by
the official $1.042311$ adjustment into 2022 dollars without clipping. Their generation bounds
are $[-20{,}844.135378,4{,}388{,}124.098445]$ and $[0,1{,}042{,}309.957689]$;
hours and weeks use $[1,99]$ and $[1,52]$, with $99$ hours denoting $99$ or more. These broad
release-code bounds, not observed extrema, also define marginal-error normalization. The eight cells cross
recorded sex with ages 18--29, 30--44, 45--59 and 60+. All three studies use equal group mass
with unweighted people within groups, not nationally representative survey-weighted estimates.
NHANES and ACS retain otherwise eligible profiles with missing non-cell demographics.
All three studies permit fractional generated values on every numeric target; the structured
grammar controls shape, keys and IDs, while strict post-parsing enforces the numeric bounds.

\paragraph{Source separation and conditioning.}
ACS assigns $60\%$ of eligible \texttt{SERIALNO} units to the source and the remainder to the
human reference by a fixed seeded hash, keeping all people from a housing unit together.
It samples at most one target per unit across all cells and excludes every target unit from
the example pool; consequently transfer fitting and evaluation targets are also
household-disjoint. For group quarters the released serial identifies a person, not a facility,
so facility separation is not claimed. The reference retains all eligible people in its units,
without the one-target-per-unit constraint; its within-cell person distribution need not match
target inclusion probabilities. NHANES and ESS use respondent-disjoint source/reference splits.
Every study fixes three complete real source examples per fifty-profile group-pure block,
excluding all sampled targets (their entire housing units in ACS). Target numeric outcomes
and human-reference records never enter prompts; missing categorical profile values are
rendered explicitly. The reference contains separate human respondents. It evaluates the population of answers,
rather than matching each generated person's answers to their observed human responses.

\begin{table}[htbp]
\centering
\caption{Request-size comparisons under constrained output formats. Entries are matched-ID pair MAE means over three
generation seeds. $\Delta_{10-1}$ is size ten minus size one. A size without a registered arm is
shown as unavailable. The validity range is the fraction of requested size-ten respondents
retained; failed arms remain included.}
\label{tab:structured_frontier}
\scriptsize
\setlength{\tabcolsep}{3pt}
\input{generated/structured_frontier}
\end{table}

\begin{figure}[htbp]
\centering
\includegraphics[width=0.98\textwidth]{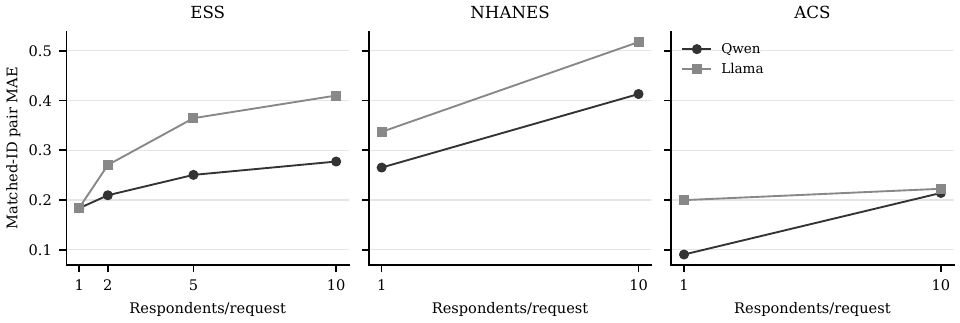}
\caption{Request-size comparisons under constrained output formats. Lines join the matched-ID pair MAE means at
the registered sizes within each domain and model. ESS has sizes $1,2,5,10$; NHANES and ACS have
sizes $1,10$. The ACS size-ten arms have six quality failures in total, retained in the plotted
means.}
\label{fig:structured_frontier}
\end{figure}

Every planned model--seed--size condition is complete. ESS seed-level error increases are
$\RelFrontierSeedsESSQwen\%$ for Qwen and $\RelFrontierSeedsESSLlama\%$ for Llama.
Section~\ref{sec:control} reports the mean changes and mixed Llama/ACS result;
Table~\ref{tab:new_domain_quality} retains all quality failures.

\begin{table}[htbp]
\centering
\caption{Quality flags for the structured request-size comparisons. The gate is at least $99\%$ retained
respondents and zero truncations for every model--seed--size arm.}
\label{tab:new_domain_quality}
\scriptsize
\setlength{\tabcolsep}{4pt}
\input{generated/new_domain_quality}
\end{table}

\subsection{Fitting corrections without human evaluation answers}

The original six variants and their fitting choices were fixed before inspection of
correction effects. Human-reference values were used only for scoring. The later single-parameter,
reference-size and transformation checks in Appendix~\ref{app:correction} are exploratory;
Appendix~\ref{app:review_diagnostics} documents them separately.

The request-size and correction studies retain profiles in different ways. The
request-size comparison in Table~\ref{tab:structured_frontier} starts with the full
planned cohort. A profile is included only if its answers are valid at every registered
size for that model and seed ($1,2,5,10$ for ESS; $1,10$ otherwise). The correction
study first separates the preassigned fitting and evaluation halves, then retains profiles
with valid answers at both size one and size ten within each half. Across the six model--seed
arms, matched request-size counts range from $1{,}974$ to $2{,}000$ for ESS, $798$ to $800$ for
NHANES and $748$ to $782$ for ACS; matched correction-evaluation counts are $988$--$1{,}000$,
$399$--$400$ and $359$--$393$, respectively. The request-size and correction studies therefore have different accepted cohorts and
different uncorrected-batch baselines. A secondary correction analysis uses every valid
ten-person evaluation row, with the same IDs across transformed variants. Matching aligns
accepted profiles but neither recovers invalid answers nor removes their quality flags.

\paragraph{Finite-sample rank conventions.}
For sorted fitting values $a_1,\ldots,a_n$ and an input $x$, let $L$ and $R$
be the numbers of fitting values strictly below and at most $x$, respectively.
The probability-integral transform is
$u=(L+U(R-L)+1/2)/(n+1)$, so $0<u<1$, including outside the fitting range.
Here $U$ is a deterministic ID-keyed uniform, obtained from the SHA-256 hash of
the fixed version, transform-domain label, variable and respondent ID; the
implementation uses the midpoint of a 52-bit bin. Separate fixed domain labels
are used for batched input ranks, transformed-score ranks and one-person fitting
ranks. The inverse empirical CDF returns $a_{1+\min(\lfloor nu\rfloor,n-1)}$.
In words, the transform first locates an answer among the fitting values. If several
fitting answers have that value, the ID-keyed draw assigns it a position within the tie.
The probability stays strictly between zero and one, so converting it to a normal score
cannot produce an infinite value. These conventions apply at both empirical-CDF steps. The batch-level marginal-preserving control instead assigns
sorted evaluation values directly, with respondent ID breaking tied scores.

Tables~\ref{tab:tradeoff_main}--\ref{tab:tradeoff_marginal_variants} report all six original
variants on the same matched evaluation profiles: uncorrected batching, marginal adjustment
alone, the full correction, within-group resampling of whole one-person fitting rows,
an identity dependence target, and held-out one-person outputs. Resampling uses group
membership but ignores other profile attributes. None of these methods guarantees agreement
with every individual profile.

\begin{table}[htbp]
\centering
\caption{Correlation and individual-variable accuracy after correction. $A$ is correlation
MAE on matched held-out profiles and $M$ is normalized marginal Wasserstein-1 error;
lower is better. Entries average three generation
seeds. Deltas are learned dependence minus untouched batch. The ACS dagger marks six size-ten
quality-failed model--seed arms, which remain included.}
\label{tab:tradeoff_main}
\scriptsize
\setlength{\tabcolsep}{3pt}
\input{generated/tradeoff_main}
\end{table}

\begin{table}[htbp]
\centering
\caption{Association-error variants. Entries are matched-evaluation means over three generation
seeds. The full correction does not always have the lowest error. The table includes every
comparison specified in the original protocol.}
\label{tab:tradeoff_pair_variants}
\scriptsize
\setlength{\tabcolsep}{2.5pt}
\input{generated/tradeoff_pair_variants}
\end{table}

\begin{table}[htbp]
\centering
\caption{Marginal-error variants. Entries are matched-evaluation means of normalized
Wasserstein-1 error over three generation seeds. The same output cohorts and locked variants as
Table~\ref{tab:tradeoff_pair_variants} are used.}
\label{tab:tradeoff_marginal_variants}
\scriptsize
\setlength{\tabcolsep}{2.5pt}
\input{generated/tradeoff_marginal_variants}
\end{table}

\paragraph{Prespecified ACS scoring sensitivities.}
The original six-variant ACS transfer also omits the component--total income/wages
pair and, separately, applies $\operatorname{asinh}(y/10{,}000)$ to income and wages
before recentering within groups. These checks change scoring while keeping the fitted corrections fixed.
All failed-quality ACS arms remain included. Table~\ref{tab:acs_sensitivities}
reports the three-seed means; these prespecified checks concern the original
transfer variants, not the subsequently designed marginal-preserving controls.

\begin{table}[htbp]
\centering
\caption{Prespecified ACS correlation-error sensitivities for all six original variants.
Lower is better. All six size-ten quality failures are retained.}
\label{tab:acs_sensitivities}
\scriptsize
\setlength{\tabcolsep}{3pt}
\input{generated/acs_prespecified_sensitivities}
\end{table}

\subsection{Quality, provenance, and verification}

Independent implementations reproduce every planned correction variant without endpoint discrepancies.
ESS verifies $21{,}600$ raw requests and $47{,}933$ accepted rows; NHANES and ACS each verify
$5{,}280$ requests, with $9{,}597$ and $9{,}413$ accepted rows. Invalid responses are never
redrawn or clipped, and quality failures remain reported. The result builder validates
completion, verification, variant grids and matched scopes against SHA-256-bound aggregates.
Public replay uses only aggregate reports, not profiles, prompts, fitted maps or individual outputs.

%% file: sections/app_marginal_controls.tex
\section{Exploratory corrections that preserve marginal distributions}\label{app:marginal_controls}

The design of these 18 September controls followed inspection of the original correction
results (Section~\ref{sec:correction}). They reuse the verified parameters, fitting
distributions, cohorts and human references. The source and protocol were fixed before
scoring, and transformed rows were saved before the references were reopened. The analysis requires no new generation
or human fitting answers, but the reused reference remains exploratory rather than a
preregistered test or new holdout.

The two variants differ in where their final answer values come from. The fit-batch
variant keeps the fitted dependence transformation, but takes final values from the
batched fitting distribution instead of the one-person fitting distribution. It can
transform each new row separately. Because the fitting and evaluation distributions
may differ, however, the evaluation marginals need not be preserved exactly.

The exact-batch variant takes its final values from the evaluation batch itself. Within
each group and column, it sorts the original values and assigns them to respondents in
the order of their corrected scores (Section~\ref{sec:correction}). The same values
therefore remain in the batch, although different respondents may receive them.

For exact preservation, reordering takes place within each group across requests,
rather than within each ten-person request. Each reported sample undergoes reordering
separately, so matching cannot break preservation. Reordering can still move values between
profiles, and exact preservation for this batch provides no guarantee for future rows.

\begin{table}[htbp]
\centering
\caption{Original and first exploratory controls: correlation MAE $A$ and range-normalized
marginal Wasserstein error $M$. Means use the same matched evaluation cohorts and three seeds.
Bold denotes minima within these eight original methods, including ties, not all later scalar
controls. Exact-batch preserves $M$. $\dagger$: six ACS size-ten quality failures retained.
SD-normalized ACS results appear in Table~\ref{tab:review_metrics_acs}.}
\label{tab:marginal_controls}
\footnotesize\setlength{\tabcolsep}{3pt}
\input{generated/correction_all_methods}
\end{table}

Figure~\ref{fig:tradeoff} in the main text shows the original correction's trade-off.

\begin{figure}[htbp]
\centering
\includegraphics[width=0.98\textwidth]{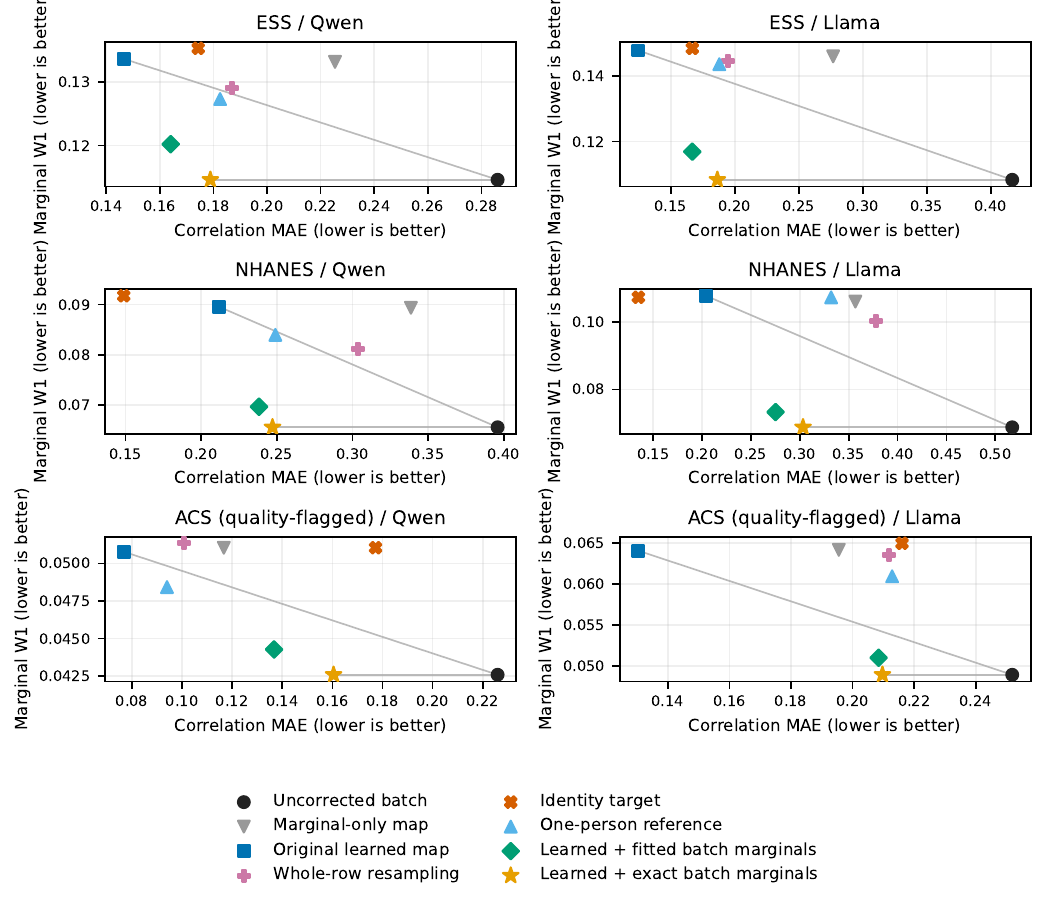}
\caption{All eight methods on the same matched evaluation cohorts, averaged across three
generation seeds. Lower is better on both axes. The exact-preservation control has the raw
batch's marginal error by construction; its correlation improvement is measured. ACS panels
retain all six failed size-ten quality arms. Means do not imply improvement in every seed.}
\label{fig:marginal_controls}
\end{figure}

\begin{table}[htbp]
\centering
\caption{Seed-level correlation-error changes relative to the same uncorrected matched batch.
Negative values improve accuracy. Exact preservation improves $17/18$ contrasts;
ACS/Llama seed 2 worsens by approximately $9\%$. $\dagger$: original ACS quality failures,
retained without regeneration.}
\label{tab:correction_seed_contrasts}
\scriptsize
\input{generated/correction_seed_contrasts}
\end{table}

All $270$ endpoints cover eight variants on matched IDs and seven on all retained batch IDs
(no unmatched one-person anchor). An independent verifier replays $42{,}608$ new transformed
rows, checks $5{,}424$ group-column marginal invariants and reproduces primary endpoints within
$3\times10^{-16}$. These checks verify the computations and add no new experimental replications. Original learned
outputs replay exactly under NumPy $2.4.6$, pandas $3.0.3$ and SciPy $1.17.1$.

Aggregate sensitivities report equal-group means of within-group Pearson and Spearman errors,
and Wasserstein distance divided by each group's human-reference SD. Correlations involving constant columns are marked undefined, with coverage counts.
Pooled-residual Pearson error remains the primary measure. Three-seed means and ranges
describe the observed runs; they are not confidence intervals. These sensitivities do
not select a variant or hyperparameter.

%% file: sections/app_followup_checks.tex
\section{Individual-variable distributions and agreement with profiles}\label{app:followup_checks}

These 19 September checks examine individual-variable accuracy and relationships between
profile attributes and answers. They reuse existing answers and human references. The
protocols, code and input hashes were fixed before computation of the new measures,
after inspection of the original results. No fitting or generation is repeated, and the reused reference
keeps the analysis exploratory.

\subsection{Marginal accuracy in the descriptive audit}\label{app:direct_marginals}

The direct audit uses the harmonized ESS/WVS and archival NHANES, ACS and MEPS panels,
not the ESS-only controlled experiment. For each variable pair, it retains finite
observations and matches cells with enough data in the human reference and every model:
at least twenty people per ESS/WVS country and fifteen per demographic cell elsewhere.
Shared cells receive equal weight. The audit calculates pooled within-cell residual
correlations and averages the two variables' within-cell Wasserstein distances. Each distance is normalized by the variable's human pooled
within-cell population standard deviation.
This uses pair-specific support and weights, unlike the original descriptive panel;
it measures within-cell marginals, not globally pooled distributions or range-normalized $M$.

\begin{table}[htbp]
\centering
\caption{Direct marginal audit on common pair-specific cells and equal cell weights.
All $738$ model--pair endpoints are supported. Entries average over pairs; cell ranges reflect
missingness. Across thirteen domain--model arms, mean marginal errors span $0.275$--$0.720$
human SDs. No equivalence threshold or test is specified; sampling error contributes.
This is not a fully crossed four-model by four-domain experiment.}
\label{tab:direct_marginals}
\scriptsize
\input{generated/direct_marginal_audit}
\end{table}

\subsection{Agreement between demographic attributes and answers}\label{app:profile_checks}

Source files verified by hash supply the categorical prompt fields for the original
human references. Group labels and all numeric values match exactly ($19{,}896$ ESS, $1{,}666$
NHANES and $15{,}085$ ACS records). The reference respondents differ from the target-profile individuals. The analysis therefore compares
distributions conditional on profile attributes, rather than each person's generated and
observed answers.

All eight variants use matched IDs across two models and three seeds. The checks test
each supplied categorical attribute separately within the original groups, excluding fields
that define those groups. Eligibility uses metadata alone: twenty human and five generated
observations in every model--seed cohort, with two eligible categories per attribute and group. This leaves nine ESS
attributes (no invented age band) and three each for NHANES and ACS; rarer combinations are
not evaluated (Table~\ref{tab:profile_coverage}).

The profile-contrast measure asks how answers differ across categories of a demographic
attribute within a group. It first compares a category's mean answer with the mean for
its whole group, then asks whether the synthetic data reproduce that difference.
For outcome $j$, category $a$ of profile field $f$, and group $c$, define
$d^P_{cfaj}=\bar y^P_{cfaj}-\bar y^P_{cj}$. The profile-contrast error is
$|d^{\mathrm{synth}}_{cfaj}-d^{\mathrm{real}}_{cfaj}|/s^{\mathrm{real}}_{W,j}$,
where the denominator is the human equal-group pooled-within-group population SD.
Absolute conditional mean error and conditional Wasserstein distance use the same scale.
Each measure averages equally over eligible category/group cells, then outcomes, then attributes.
Groups with more supported categories receive more endpoint weight; only the SD gives
groups equal mass. Attributes overlap, and coverage is limited: citizenship covers $3/20$
ESS countries and race $4/8$ ACS groups.

Original and exact-preservation corrections improve $1{,}068$ and $893$ of $1{,}188$
model--seed--attribute--outcome contrasts, respectively. These overlapping comparisons are descriptive counts rather than independent
significance tests. They do not establish full profile consistency or accuracy for individuals
and small subgroups.

\begin{table}[htbp]
\centering
\caption{Profile-contrast error normalized by human pooled within-group SD, averaged over
three seeds. Relative to untouched batching, the original and exact-preservation maps reduce
domain--model means by approximately $4$--$39\%$ and $1$--$10\%$, respectively; neither is
uniformly best. All methods use the same metadata-selected support. Lower is better.
$\dagger$: existing ACS quality failures retained.}
\label{tab:profile_contrasts}
\scriptsize
\input{generated/profile_contrasts}
\end{table}

\begin{table}[htbp]
\centering
\caption{Conditional marginal Wasserstein distance within supported profile categories,
normalized by the same human pooled within-group SD. The original map raises every
domain--model mean; exact preservation lowers every mean. The latter is measured, not
guaranteed by preserving group marginals. Lower is better.}
\label{tab:profile_wasserstein}
\scriptsize
\input{generated/profile_wasserstein}
\end{table}

\begin{table}[htbp]
\centering
\caption{Absolute conditional mean error on the same profile support and human SD scale.}
\label{tab:profile_means}
\scriptsize
\input{generated/profile_means}
\end{table}

\begin{table}[htbp]
\centering
\caption{Metadata-selected profile support, common across all models, seeds and variants.
Minimum generated counts refer to one model--seed evaluation cohort, not pooled draws.}
\label{tab:profile_coverage}
\scriptsize
\input{generated/profile_coverage}
\end{table}

Both extensions verify Wasserstein distances by independently integrating CDF differences;
the direct audit also checks covariance with explicit equal-group row weights. Public outputs
contain aggregates, coverage and provenance hashes, not individual profiles or answers.

%% file: sections/app_prefix_intervention.tex
\section{Changing earlier answers in a replayed completion}\label{app:prefix_intervention}

This supporting Qwen3.8-27B/NHANES experiment tests whether earlier answers influence
BMI--sedentary-minute correlations in later responses. It changes those answers while
holding the target profile and token position fixed. The design and analysis were fixed
after a technical pilot and before scientific generation. The endpoint measures generated
relationships, without evaluating accuracy against human data.

\paragraph{Source-prefix replay.}
The supplied histories are the literal first nine answer objects from forty valid
ten-person requests at seed zero: five donor requests per demographic cell from the correction's fitting partition. A new same-cell evaluation profile
replaces only the tenth requested profile; the first nine profiles, three real examples,
instructions and five-item battery remain fixed. The original assistant preamble and nine-row
prefix are supplied as native token IDs, without another role or instruction. Independent
unfinished-assistant rendering matches these IDs exactly, including the non-thinking preamble
and trailing whitespace. A fixed-ID suffix grammar generates only the tenth answer and closing
bracket. Non-seed decoding settings are unchanged; requests receive distinct decoder seeds.

The experiment compares three versions of the supplied nine answers. The \emph{natural} condition
uses them exactly as generated. The \emph{permuted} condition moves sedentary-minute
values between respondents according to one hash-fixed permutation. The \emph{aligned}
condition instead assigns larger sedentary values to respondents with higher BMI, breaking
ties by ID. Only the assignment of sedentary values changes. The substitution preserves the original
numeric spellings, every variable's list of values, all other fields, row order, whitespace,
prompt and total context-token count. All forty
donors pass prespecified eligibility checks: nonconstant focal columns and strictly higher
prefix correlation under alignment than permutation. The permutation is fixed without searching alternatives or inspecting generated target answers. Input-prefix correlations average
$\PrefixSourceNatural$, $\PrefixSourcePermuted$ and $\PrefixSourceAligned$, respectively;
these are manipulation checks, not target outcomes.

The original prefixes were generated with a different tenth profile. Replaying them therefore
differs from generating a fresh batch or reproducing its original random sequence. Re-pairing
also changes agreement with donor profiles and relationships beyond the focal pair. The
experiment identifies the effect of this content change, without isolating covariance as
the channel through which it acts.

\paragraph{Allocation and estimator.}
The technical pilot contains twelve targets, three conditions and two repetitions. All
$72$ answers pass validation without truncation. Pilot effect sizes do not select the design. A frozen hash order then selects
$25$ non-pilot evaluation profiles per cell, five per donor prefix. Crossing $200$ targets,
three conditions and three repetitions gives $1{,}800$ new answers. Injected rows never enter
the endpoint. Raw outputs, native token IDs and commitments are retained; an independent check
decodes tokens and parses both full arrays and target suffixes. Completed invalid outputs are
not rerun. The threshold is $99\%$ validity and no truncation in each condition.

The primary comparison uses target--repetition cases accepted in all three conditions.
Within each cell, pooled targets and repetitions provide the observations for calculating
Pearson correlations. The endpoint then averages the aligned-minus-permuted difference
equally over eight cells. This calculation
averages correlations rather than the covariances in Equation~\ref{eq:within}, and remains
conditional on joint acceptance.

The $5{,}000$ paired, cell-stratified bootstrap draws resample five donor blocks per cell,
preserving each donor's complete matched targets, repetitions and conditions with multiplicity;
individuals are not resampled again. The frozen rule reports a percentile interval only if the
primary correlations and at least $99\%$ of draws are defined. With five donors per cell,
nominal $95\%$ coverage is not established: $600$ responses per arm are not $600$ independent units.

\paragraph{Results and scope.}
All $\PrefixValidResponses/1{,}800$ answers are valid, with $\PrefixTruncations$ truncations and
$\PrefixMatchedTriplets$ matched triplets; matched and all-retained results coincide.
The contrast is $\PrefixDelta$, with donor-block interval
$[\PrefixIntervalLow,\PrefixIntervalHigh]$ and $\PrefixBootstrapValid/5{,}000$ estimable draws.
Repetition contrasts are $\PrefixRepZero$, $\PrefixRepOne$ and $\PrefixRepTwo$; all $24$
cell--repetition contrasts are positive. Centering outcomes within donor blocks gives
$\PrefixDonorCentered$, so between-donor means alone do not explain the contrast.

\begin{table}[h]
\centering
\caption{BMI--sedentary correlations among newly generated tenth respondents only;
$\Delta$ is aligned minus permuted history. Injected records are excluded.}
\label{tab:prefix_intervention}
\small
\input{generated/prefix_intervention}
\end{table}

Changing the earlier answers changes later relationships, but does not establish greater
accuracy or show how much of the batching penalty this mechanism explains. Keeping the
same values in the prefix also does not guarantee agreement with profiles or preserve
the distribution of new target answers. Alignment raises mean sedentary time by $9.446$
minutes relative to permutation; mean Wasserstein distance between conditions is $32.963$
minutes across cells. These measure differences between generated outputs, not errors
against human data. The supplement also reports contrasts with natural history, all-retained
results and marginal changes for all five outcomes. These are descriptive checks without
multiplicity adjustment.